\documentclass{article}

\usepackage[preprint,nonatbib]{nips_2018}

\usepackage[utf8]{inputenc} 
\usepackage[T1]{fontenc}    
\usepackage{hyperref}       
\usepackage{url}            
\usepackage{booktabs}       
\usepackage{amsfonts}       
\usepackage{nicefrac}       
\usepackage{microtype}      

\usepackage{amsmath,amsthm,amssymb}
\usepackage{mathtools}
\usepackage{algorithm,algorithmic}
\usepackage{import}

\usepackage{subcaption}
\usepackage{wrapfig}

\newcommand{\reals}{\mathbb{R}}

\newcommand{\Hilb}{\mathcal{H}}

\newcommand{\KL}{\operatorname{KL}}

\newcommand{\zero}{\mathbf{0}}

\newcommand{\trans}{\intercal}

\newcommand{\Proj}{\operatorname{P}}
\newcommand{\prox}{\operatorname{prox}}
\newcommand{\dom}{\operatorname{dom}}

\newcommand{\interior}{\operatorname{int}}

\newcommand{\Ord}{\mathcal{O}}

\newcommand{\diverge}{\operatorname{div}}

\renewcommand{\Pr}{\operatorname{Pr}}
\newcommand{\expect}{\operatorname{\mathbf{E}}}

\newcommand{\normal}{\mathcal{N}}

\newcommand{\bigbrace}[1]{\left\{\begin{array}{lr} #1 \end{array} \right.}
\newcommand{\otherwise}{\text{otherwise}}

\newcommand{\lr}[1]{\left( #1 \right)}

\DeclareMathOperator*{\argmin}{argmin}
\DeclareMathOperator*{\argmax}{argmax}

\makeatletter
\newtheorem*{rep@theorem}{\rep@title}
\newcommand{\newreptheorem}[2]{%
\newenvironment{rep#1}[1]{%
 \def\rep@title{#2 \ref{##1}}%
 \begin{rep@theorem}}%
 {\end{rep@theorem}}}
\makeatother

\newtheorem{proposition}{Proposition}
\newreptheorem{proposition}{Proposition}

\theoremstyle{definition}

\newcommand{\Wass}{\mathcal{W}}

\newcommand{\Xspace}{\mathcal{X}}
\newcommand{\Yspace}{\mathcal{Y}}
\newcommand{\Pspace}{\mathcal{P}}

\newcommand{\vvec}{\mathbf{v}}
\newcommand{\wvec}{\mathbf{w}}

\newcommand{\Amat}{\mathbf{A}}

\newcommand{\bvec}{\mathbf{b}}

\newcommand{\xvec}{\mathbf{x}}
\newcommand{\yvec}{\mathbf{y}}

\newcommand{\kernel}{\kappa}

\newcommand{\pdf}{\operatorname{\rho}}

\newcommand{\primal}{P}
\newcommand{\dual}{D}

\newcommand{\dualintegrand}{d}

\newcommand{\xrv}{X}
\newcommand{\yrv}{Y}

\newcommand{\Wiener}{\mathbf{W}}

\title{Approximate inference with Wasserstein \\ gradient flows}

\author{
  Charlie Frogner \\
  Brain and Cognitive Sciences\\
  Massachusetts Institute of Technology\\
  \texttt{frogner@mit.edu} \\
  \And
  Tomaso Poggio \\
  Brain and Cognitive Sciences\\
  Massachusetts Institute of Technology\\
  \texttt{tp@ai.mit.edu} \\
}

\begin{document}

\maketitle

\begin{abstract}
We present a novel approximate inference method for diffusion processes, based on the Wasserstein gradient flow formulation of the diffusion. In this formulation, the time-dependent density of the diffusion is derived as the limit of implicit Euler steps that follow the gradients of a particular free energy functional. Existing methods for computing Wasserstein gradient flows rely on discretization of the domain of the diffusion, prohibiting their application to domains in more than several dimensions. We propose instead a discretization-free inference method that computes the Wasserstein gradient flow directly in a space of continuous functions. We characterize approximation properties of the proposed method and evaluate it on a nonlinear filtering task, finding performance comparable to the state-of-the-art for filtering diffusions.
\end{abstract}

\section{Introduction}

Diffusion processes are ubiquitous in science and engineering. They arise when modeling dynamical systems driven by random fluctuations, such as action potentials in neuroscience, interest rates and asset prices in finance, reaction dynamics in chemistry, population dynamics in ecology, and in numerous other settings. In signal processing and machine learning, diffusion processes provide the dynamics underlying classic filtering methods such as the Kalman filter \cite{Kalman:1961wy}.

Inference for general diffusions is an outstanding challenge. Each diffusion process defines a probability distribution that evolves in continuous time; inference involves solving for the distribution at a future time given an initial distribution at the current time. Exact, closed-form solutions are typically unavailable, and numerous approximations have been proposed, including parametric approximations \cite{Kalman:1961wy} \cite{Julier:1995fj}, particle or sequential Monte Carlo methods \cite{Crisan:1999ub} \cite{Fearnhead:2008ui}, MCMC methods \cite{Roberts:2001vs} \cite{Golightly:2008cx} and variational approximations \cite{Archambeau:2007te} \cite{Vrettas:2015tu} \cite{Sutter:2016wi}. Each poses a different tradeoff between fidelity of the approximation and computational burden.

In this paper, we investigate a novel approximate inference method for nonlinear diffusions. It is based on a characterization, due to Jordan, Kinderlehrer and Otto \cite{Jordan:1998bf}, of the diffusion process as following a {\bf gradient flow} with respect to a Wasserstein metric on probability measures. Concretely, they define a time discretization of the diffusion process in which the approximate probability density $\rho_k$ at the $k$th timestep solves a variational problem,
\begin{equation}
\label{eq:wassflow-intro-wass-gradient-step}
\rho_k = \argmin_{\rho \in \Pspace(\Xspace)} \Wass_2^2(\rho, \rho_{k-1}) + 2 \tau f(\rho)
\end{equation}
with $\Wass_2 : \Pspace(\Xspace) \times \Pspace(\Xspace) \rightarrow \reals$ being the $2-$Wasserstein distance, $f : \Pspace(\Xspace) \rightarrow \reals$ a free energy functional defining the diffusion process, and $\tau > 0$ the size of the timestep \footnote{$\Pspace(\Xspace)$ is the space of probability measures defined on domain $\Xspace$.}. This discrete process is shown to converge, as $\tau \rightarrow 0$, to the exact diffusion process.

For reasonable values of the timestep $\tau$, the time-discretized Wasserstein gradient flow in \eqref{eq:wassflow-intro-wass-gradient-step} gives a close approximation to the density of the diffusion. In Figure \ref{fig:wassflow-intro-example-exact}, we use the method described in this paper (Sections \ref{sec:wassflow-smoothed-dual} and \ref{sec:wassflow-inference-via-stochastic-programming}) to compute the Wasserstein gradient flow for a simple diffusion, initialized with a bimodal density. We see that it follows the exact density closely.

Exact computation of the time-discretized gradient step in \eqref{eq:wassflow-intro-wass-gradient-step} is intractable in general. Existing numerical methods rely on discretization of the domain of the diffusion, which restricts their application to spaces with very few dimensions -- typically three or fewer. In this work, we propose a novel method for computing the gradient flow that avoids discretization, opting instead to operate directly on continuous functions lying in a reproducing kernel Hilbert space. Specifically, we derive a dual problem to \eqref{eq:wassflow-intro-wass-gradient-step} that uses a regularized Wasserstein distance in place of the unregularized one in \eqref{eq:wassflow-intro-wass-gradient-step}. We show that, for a general strictly convex, smooth regularizer, this dual problem is an unconstrained stochastic program, which admits a tractable finite-dimensional RKHS approximation. This approach is motivated by a similar observation for the case of entropic regularization of optimal transport in \cite{Genevay:2016stochastic}. Our proposed approximation yields an approximate inference method for diffusions that is computationally tractable in settings where domain discretization is impractical.

The rest of this paper is organized as follows. In Section 2 we review diffusion processes and discuss related work. In Section 3 we derive a smoothed dual formulation of the Wasserstein gradient flow, and in Section 4 we use this dual formulation to derive a novel inference algorithm. In Section 5 we investigate theoretical properties. In Section 6 we characterize empirical performance of the proposed algorithm, before concluding.

\begin{figure}[t]
\centering
\begin{subfigure}[t]{0.32\textwidth}
  \centering
  \includegraphics[width=\textwidth]{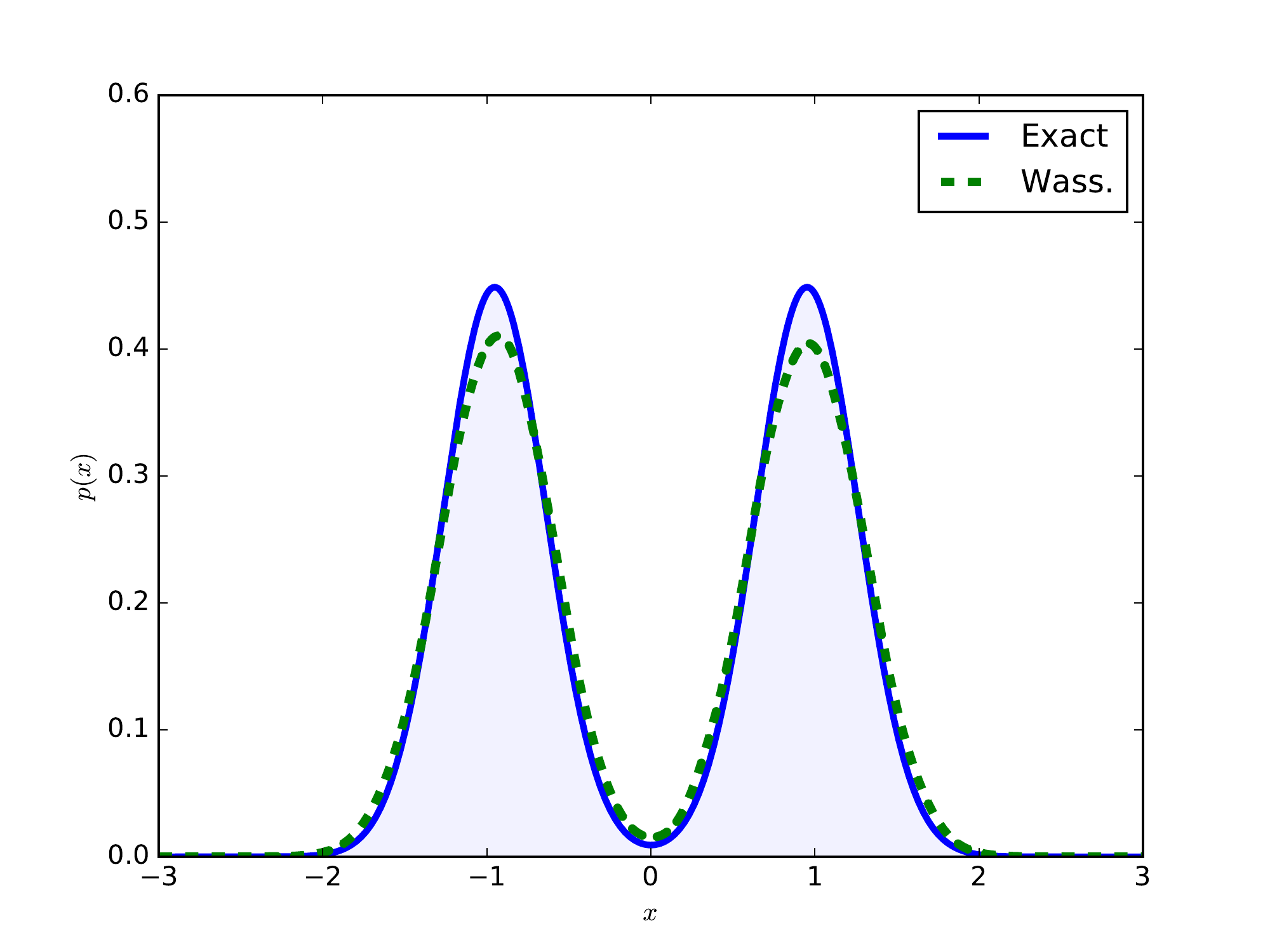}
  \caption{$t = 0.05$.}
  \label{fig:wassflow-intro-example-exact-density-0}
\end{subfigure}
\begin{subfigure}[t]{0.32\textwidth}
  \centering
  \includegraphics[width=\textwidth]{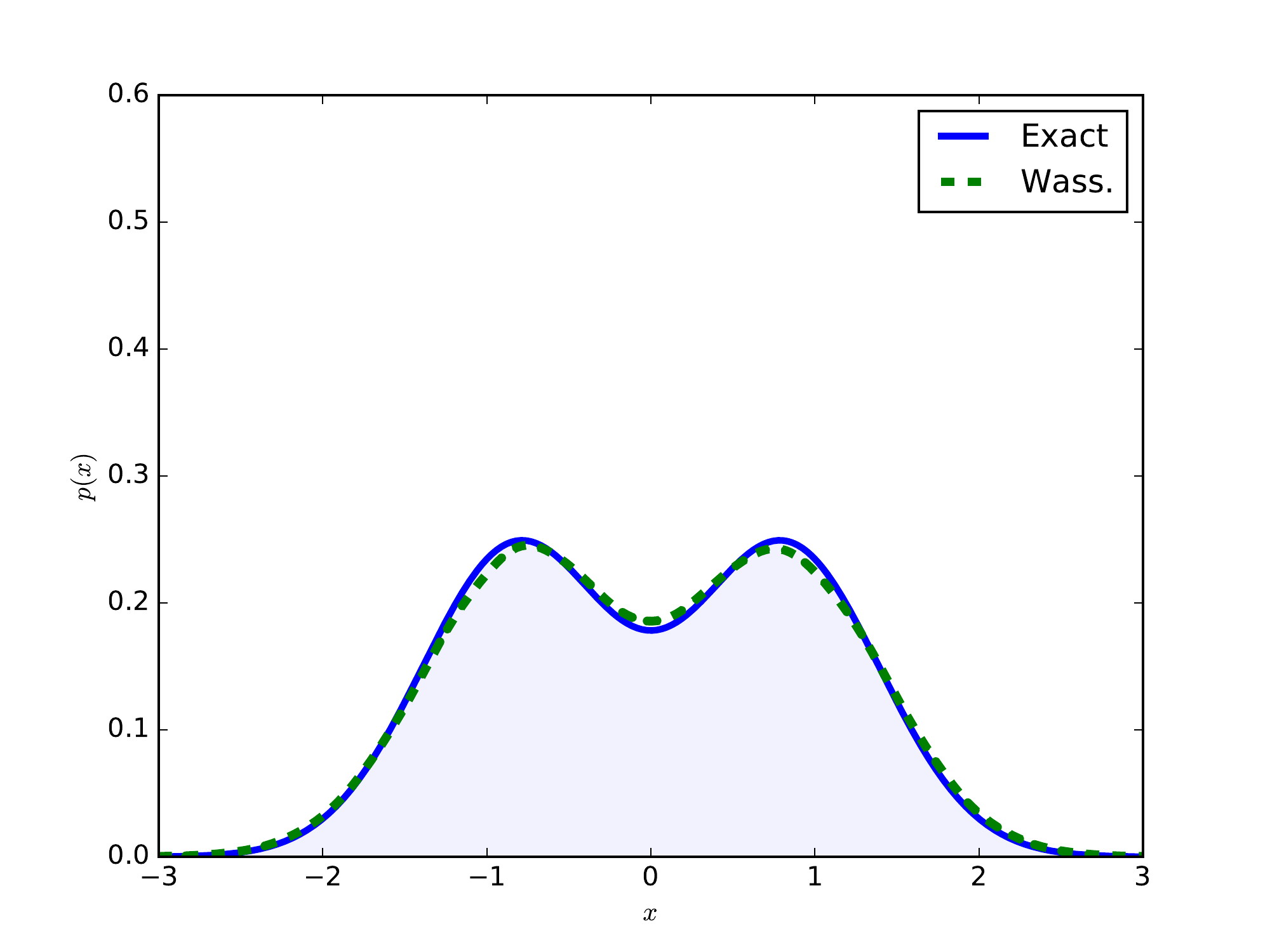}
  \caption{$t = 0.2$}
  \label{fig:wassflow-intro-example-exact-density-1}
\end{subfigure}
\begin{subfigure}[t]{0.32\textwidth}
  \centering
  \includegraphics[width=\textwidth]{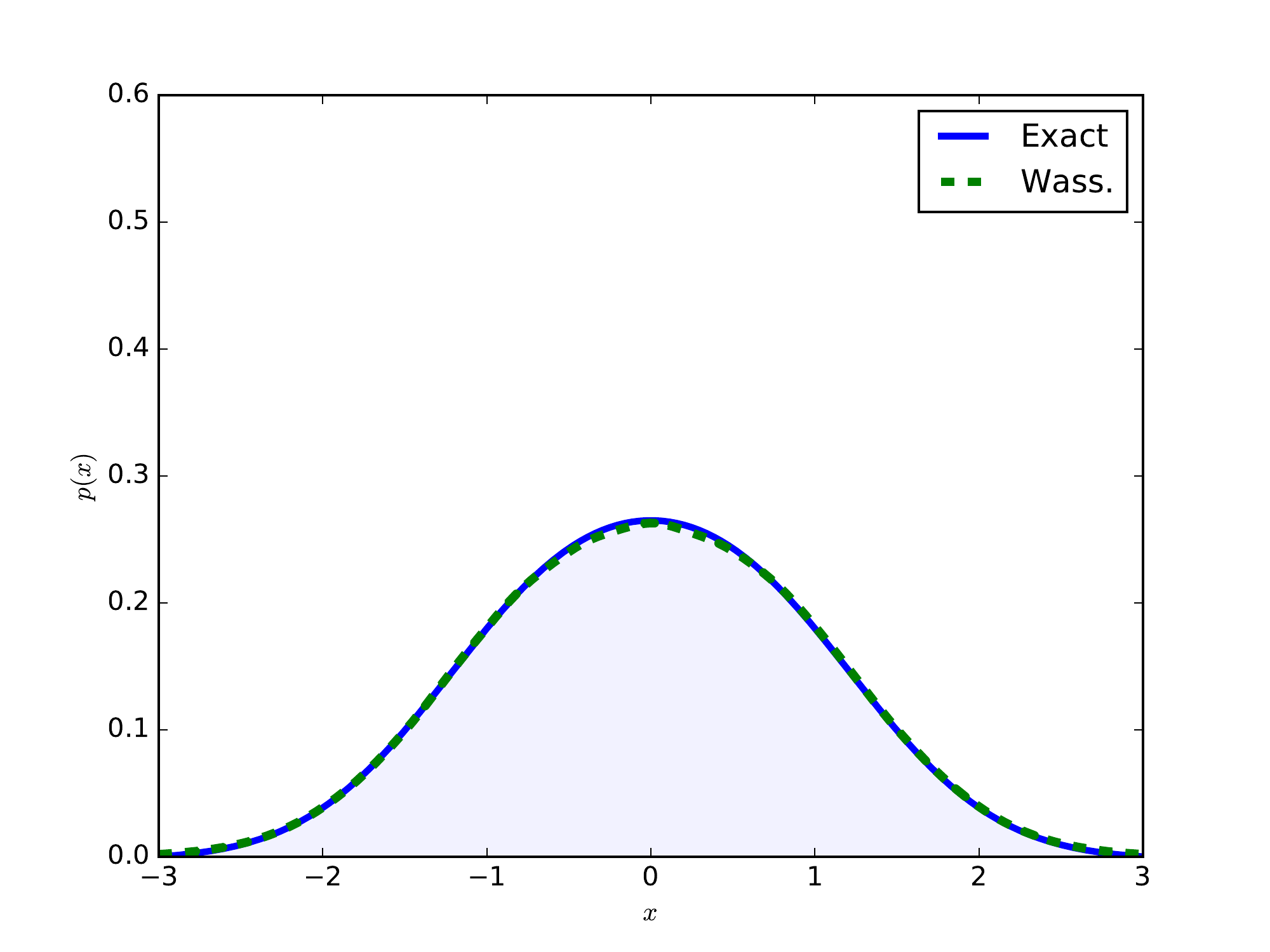}
  \caption{$t = 0.5$.}
  \label{fig:wassflow-intro-example-exact-density-2}
\end{subfigure}
\caption{Regularized Wasserstein gradient flow (Section \ref{sec:wassflow-smoothed-dual}) approximates closely an Ornstein-Uhlenbeck diffusion, initialized with a bimodal density. Both the regularization ($\gamma$) and the discrete timestep ($\tau$) are sources of error. Shaded region is the true density.}
\label{fig:wassflow-intro-example-exact}
\end{figure}

\section{Background and related work}

\subsection{Notation}

$\Xspace$ is a smooth manifold. $\mathcal{M}_+(\Xspace)$ is the set of nonnegative Radon measures on $\Xspace$ and $\Pspace(\Xspace)$ is the set of probability measures on $\Xspace$, $\Pspace(\Xspace) = \{ \mu \in \mathcal{M}_+(\Xspace) | \mu(\Xspace) = 1 \}$. Given a joint probability measure $\pi$ on the product space $\Xspace \times \Xspace$, its marginals are the measures $\Proj_1 \pi \in \Pspace(\Xspace)$ and $\Proj_2 \pi \in \Pspace(\Xspace)$ defined by
$$ (\Proj_1 \pi)(A) = \pi(A \times \Yspace), \quad (\Proj_2 \pi)(B) = \pi(\Xspace \times B), $$
for $A \subseteq \Xspace$ and $B \subseteq \Yspace$ measurable subsets of $\Xspace$ and $\Yspace$. $\reals_+$ is the set of nonnegative reals, while $\reals_{++}$ are positive reals.

\subsection{Diffusions, free energy, and the Fokker-Planck equation}
\label{sec:bg-diffusions}

We consider a continuous-time stochastic process $\xrv_t$ taking values in a smooth manifold $\Xspace$, for $t \in [t_i, t_f]$, and having single-time marginal densities $\pdf_t : \Xspace \rightarrow \reals$ with respect to a reference measure on $\Xspace$. We are specifically interested in diffusion processes whose single-time marginal densities obey a diffusive partial differential equation,
\begin{equation}
\label{eq:wassflow-bg-diffusive-PDE}
\frac{\partial \pdf_t}{\partial t} = \diverge \left[\pdf_t \nabla f^{\prime}(\pdf_t)\right],
\end{equation}
with $f : \Pspace(\Xspace) \rightarrow \reals$ a functional on densities and $f^{\prime}$ its gradient for the $L^2(\Xspace)$ metric.

$f$ is the \emph{free energy} and defines the diffusion entirely. An important example, which will be our primary focus, is the \emph{advection-diffusion process}, which is typically characterized as obeying an It\^o stochastic differential equation,
\begin{equation}
\label{eq:wassflow-bg-advection-diffusion-sde}
d\xrv_t = -\nabla w(\xrv_t) dt + \beta^{-1/2} d\Wiener_t
\end{equation}
with $\nabla w$ being the gradient of a potential function $w : \Xspace \rightarrow \reals$, determining the advection or drift of the system, and $\beta^{-1/2} > 0$ the magnitude of the diffusion, which is driven by a Wiener process having stochastic increments $d\Wiener_t$ (see \cite{Kloeden:2013vt} for a formal introduction) \footnote{We assume sufficient conditions for existence of a strong solution to \eqref{eq:wassflow-bg-advection-diffusion-sde} are fulfilled \cite{Oksendal:2013ug} Thm. 5.2.1.}. The advection-diffusion has marginal densities obeying a \emph{Fokker-Planck} equation,
\begin{equation}
\label{eq:wassflow-bg-advection-diffusion-kolmogorov-forward}
\frac{\partial \pdf_t}{\partial t} = \beta^{-1} \Delta \pdf_t + \diverge[\pdf_t \nabla w],
\end{equation}
which is a diffusive PDE with free energy functional $f(\pdf) = \langle w, \pdf \rangle_{L^2(\Xspace)} + \beta^{-1} \langle \pdf, \log \pdf \rangle_{L^2(\Xspace)}$, for scalar potential $w \in L^2(\Xspace)$.
The advection-diffusion is \emph{linear} whenever $\nabla w$ is linear in its argument.

We note that the current work applies to those diffusions that can be rendered into the form \eqref{eq:wassflow-bg-diffusive-PDE} via a change of variables. In particular, in the case of advection-diffusion, these are the \emph{reducible} diffusions and include nearly all diffusions in one dimension \cite{AitSahalia:2008jb}.

\subsection{Approximate inference for diffusions}

Inference for a nonlinear diffusion is generally intractable. Given an initial density at time $t_i$, the goal is to determine the single-time marginal density $\pdf_t$ at some time $t > t_i$. Exact inference entails solving the foward PDE \eqref{eq:wassflow-bg-diffusive-PDE}, for which closed-form solutions are seldom available.

{\bf Domain discretization.} In certain cases, an Eulerian discretization of the domain, i.e. a fixed mesh, is available. Here one can apply standard numerical integration methods such as Chang and Cooper's \cite{Chang:1970wj} or entropic averaging \cite{Pareschi:2017structure} for integrating the Fokker-Planck PDE. A number of Eulerian methods have been proposed for Wasserstein gradient flows, as well, including finite element \cite{burger2009mixed} and finite volume methods \cite{carrillo2015finite}. Entropic regularization of the problem yields an efficient iterative method \cite{Peyre:2015jc}. Lagrangian discretizations, which follow moving particles or meshes, have also been explored \cite{carrillo2009numerical} \cite{westdickenberg2010variational} \cite{budd2013monge} \cite{benamou2016discretization}.

{\bf Particle simulation.} One approach to inference approximates the target density by a weighted sum of delta functions, $\rho_t(\xvec) = \sum_{i=1}^N \wvec_i \delta_{\xvec_t^{(i)} = \xvec}$, at locations $\xvec_t^{(i)} \in \Xspace$. Each delta function represents a ``particle,'' and can be obtained by sampling an initial location $\xvec_{t_i}$ according to $\rho_{t_i}$, then forward simulating a trajectory from that location, according to the diffusion. Standard simulation methods such as Euler-Maruyama discretize the time interval $[t_i, t]$ and update the particle's location recursively \cite{Kloeden:2013vt}. For a fixed time discretization, such methods are biased in the sense that, with increasing number of particles, they converge only to an approximation of the true predictive density. To address this, one can use a rejection sampling method \cite{Beskos:2005exact} \cite{Beskos:2008vt} to sample exactly (with no bias) from the distribution over trajectories. Density estimation can be used to extrapolate the inferred density beyond the particle locations \cite{Durham:2002ws} \cite{Hurn:2003it}.

{\bf Parametric approximations.} One can also approximate the predictive density by a member of a parametric class of distributions. This parametric density might be chosen by matching moments or another criterion. The extended Kalman filter \cite{Kalman:1961wy} \cite{Kushner:1967du}, for example, chooses a Gaussian density whose mean and covariance evolve according to a first order Taylor approximation of the dynamics. Sigma point methods such as the unscented Kalman filter \cite{Julier:1995fj} \cite{Julier:2000cb} \cite{Sarkka:2007td} select a deterministic set of points $\xvec_t^{(i)} \in \Xspace$ that evolve according to the exact dynamics of the process, such that the mean and covariance of the true predictive density is well-approximated by finite sums involving only these points. The mean and covariance so computed then define a Gaussian approximation. Gauss-Hermite \cite{Singer:2008ur}, Gaussian quadrature and cubature methods \cite{Sarkka:2012continuous} \cite{Sarkka:2013uo} correspond to different mechanisms for choosing the sigma points $\xvec_t^{(i)}$.

Beyond Gaussian approximations, mixtures of Gaussians have been used as well to approximate the predictive density \cite{Alspach:1972to} \cite{Terejanu:2008tc} \cite{Terejanu:2011ww}. Variational methods attempt to minimize a divergence between the chosen approximate density and the true predictive density. These can include Gaussian approximations \cite{Archambeau:2007te} \cite{AlaLuhtala:2015vl} as well as more general exponential families and mixtures \cite{Vrettas:2015tu} \cite{Sutter:2016wi}. And for a broad class of diffusions, closed-form series expansions are available \cite{AitSahalia:2008jb}.

\section{Smoothed dual formulation for Wasserstein gradient flow}
\label{sec:wassflow-smoothed-dual}

Our target is the predictive distribution of a diffusion: given an initial density $\pdf_t$, we want to evolve it forward by a time increment $\Delta t$, to obtain the solution for the diffusion \eqref{eq:wassflow-bg-diffusive-PDE} at time $t + \Delta t$. We propose to approximate this by $m$ steps of the Wasserstein gradient flow \eqref{eq:wassflow-intro-wass-gradient-step}, with stepsize $\tau = \Delta t / m$. The problem is to compute approximately this gradient step.

\subsection{Regularized Wasserstein gradient flow}

We start by introducing a proximal operator for the gradient step, which uses a regularized Wasserstein distance. For measures $\mu, \nu \in \Pspace(\Xspace)$, we define the squared, regularized $2-$Wasserstein distance as
\begin{equation}
\label{eq:wassflow-inference-wasserstein-regularized}
\Wass_{\gamma}^2(\mu, \nu) = \min_{\pi \in \Pi(\mu, \nu)} \int_{\Xspace \times \Xspace} d^2(\xvec, \yvec) d\pi(\xvec, \yvec) + \gamma R(\pi).
\end{equation}
with $d : \Xspace \times \Xspace \rightarrow [0, +\infty)$ the distance in $\Xspace$, $\Pi(\mu, \nu)$ the set of joint measures on $\Xspace \times \Xspace$ having marginals $\mu$ and $\nu$, and $R : \Pspace(\Xspace \times \Xspace) \rightarrow \reals$ a regularizer. We assume $R$ is Legendre-type (Bauschke and Borwein Def. 2.8 \cite{Bauschke:1997legendre}), implying it is closed, strictly convex, smooth, and proper. We also assume $R$ is separable, in the sense that
\begin{equation}
R(\pi) = \int_{\Xspace \times \Xspace} \bar{R}(d\pi(\xvec, \yvec)),
\end{equation}
for $\bar{R} : \reals \rightarrow \reals$ the component function. In the case of an entropy regularizer, for example, this is $\bar{R}: u \mapsto u (\log u - 1)$. For an $L^2$ regularizer, this is $\bar{R} : u \mapsto u^2$.

Given a free energy functional $f$ (Section \ref{sec:bg-diffusions}), we define the primal objective $\primal_{\nu}^{\gamma, \tau} : \Pspace(\Xspace) \rightarrow [0, +\infty)$,
\begin{equation}
\label{eq:wassflow-inference-primal-objective}
\primal_{\nu}^{\gamma, \tau}(\mu) \triangleq \Wass_{\gamma}^2(\mu, \nu) + 2 \tau f(\mu),
\end{equation}
for $\gamma \geq 0$, and $\tau > 0$. The primal formulation for the regularized Wasserstein gradient flow is 
\begin{equation}
\label{eq:wassflow-inference-proximal-operator-smoothed}
\prox_{\tau f}^{\Wass_{\gamma}} \nu = \argmin_{\mu \in \Pspace(\Xspace)} \primal_{\nu}^{\gamma, \tau}(\mu).
\end{equation}
For $\gamma > 0$, the map $\mu \mapsto \Wass_{\gamma}(\mu, \nu)$ is strictly convex and coercive such that, assuming a convex functional $f$ in \eqref{eq:wassflow-inference-primal-objective}, the proximal operator is uniquely defined. 

\begin{wrapfigure}[9]{R}{0.49\textwidth}
\begin{minipage}{0.49\textwidth}
\caption{Free energy expressions for advection-diffusion}
\label{tbl:wassflow-diffusion-free-energy-expressions}
\centering
  \begin{tabular}{l}
    \toprule
    $f(\mu) = \langle w, d\mu \rangle + \beta^{-1} \langle d\mu, \log d\mu - 1\rangle$ \\
    $f^{\ast}(z) = \beta^{-1} \int_{\Xspace} \exp\lr{\beta (z(\xvec) - w(\xvec))}$ \\
    $\lr{\nabla f^{\ast}(z)}(\xvec) = \exp\lr{\beta (z(\xvec) - w(\xvec))}$ \\
    $\lr{\nabla^2 f^{\ast}(z)}(\xvec) = \beta \exp\lr{\beta (z(\xvec) - w(\xvec))}$ \\
    \bottomrule
  \end{tabular}
\end{minipage}
\end{wrapfigure}

Note that we give all formulas in terms of a general free energy $f$. Table \ref{tbl:wassflow-diffusion-free-energy-expressions} gives concrete expressions for the free energy and its conjugate, in the case of an advection-diffusion system.

\subsection{Smoothed dual formulation}

Computing the proximal operator \eqref{eq:wassflow-inference-proximal-operator-smoothed} directly entails solving an infinite program over the set of possible joint measures $\pi \in \Pspace(\Xspace \times \Xspace)$ having $\nu$ as the second marginal. As a step towards a tractable approximation, we will derive a dual formulation that is unconstrained.

The dual objective $\dual_{\nu}^{\gamma, \tau} : L^2(\Xspace) \times L^2(\Xspace) \rightarrow \reals$ is
\begin{equation}
\label{eq:wassflow-inference-dual-objective}
\dual_{\nu}^{\gamma, \tau}(g, h) \triangleq -\tau f^{\ast}\lr{-\frac{1}{\tau} g} + \langle h, d\nu \rangle_{L^2(\Xspace)} - \gamma R^{\ast}\lr{\max\left\{\frac{1}{\gamma}\lr{g + h - d^2}, \nabla R(\zero)\right\}},
\end{equation}
with $f^{\ast}$ and $R^{\ast}$ the convex conjugates \footnote{$f^{\ast}(z) = \sup_{\mu} \langle \mu, z \rangle_{L^2(\Xspace)} - f(\mu)$, $R^{\ast}(\xi) = \sup_{\pi} \langle \pi, \xi \rangle_{L^2(\Xspace \times \Xspace)} - R(\pi)$.}.
We have the following.

\begin{proposition}[Strong duality]
\label{prop:wassflow-inference-fenchel-dual}
Let $\nu \in \Pspace(\Xspace)$ and $f : \Pspace(\Xspace) \rightarrow [0, +\infty)$ a convex, lower semicontinuous and proper functional. Define $\primal_{\nu}^{\gamma, \tau}$ as in \eqref{eq:wassflow-inference-primal-objective} and $\dual_{\nu}^{\gamma, \tau}$ as in \eqref{eq:wassflow-inference-dual-objective}. Assume $\gamma > 0$. Then
\begin{equation}
\label{eq:wassflow-inference-fenchel-dual}
\min_{\mu \in \Pspace(\Xspace)} \primal_{\nu}^{\gamma, \tau}(\mu) = \max_{g \in L^2(\Xspace), h \in L^2(\Xspace)} \dual_{\nu}^{\gamma, \tau}(g, h).
\end{equation}
Suppose $f$ is strictly convex and let $g_{\ast}, h_{\ast}$ maximize $\dual_{\nu}^{\gamma, \tau}$. Then
\begin{equation}
\label{eq:wassflow-inference-fenchel-duality-solution-relation}
\mu_{\ast} = \nabla f^{\ast}(-\frac{1}{\tau} g_{\ast})
\end{equation}
minimizes $\primal_{\nu}^{\gamma, \tau}$.
\end{proposition}

Importantly, we have replaced the linearly-constrained optimization in the primal \eqref{eq:wassflow-inference-proximal-operator-smoothed} with an unconstrained problem \eqref{eq:wassflow-inference-fenchel-dual}.

\section{Inference via stochastic programming}
\label{sec:wassflow-inference-via-stochastic-programming}

\subsection{Stochastic programming formulation}

The unconstrained dual problem \eqref{eq:wassflow-inference-dual-objective} is not directly computable in general. To construct an approximation, we start by noting that the dual has an interpretation as a stochastic program. Specifically, let $\mu_0, \nu_0 \in \Pspace(\Xspace)$ be arbitrarily chosen probability measures, supported everywhere in $\Xspace$. We can express the dual objective \eqref{eq:wassflow-inference-dual-objective} as 
\begin{equation}
\label{eq:wassflow-continuous-expectation-max-dual}
\dual_{\nu}^{\gamma, \tau}(g, h) = \expect_{\xrv, \yrv} \dualintegrand_{\nu}^{\gamma, \tau}(\xrv, \yrv, g, h)
\end{equation}
for random variables $\xrv, \yrv$ distributed as $\mu_0$ and $\nu_0$, respectively, where the integrand $\dualintegrand_{\nu}^{\gamma, \tau}$ is
\begin{align}
\begin{split}
\label{eq:wassflow-continuous-expectation-max-dual-integrand}
\dualintegrand_{\nu}^{\gamma, \tau}(\xvec, \yvec, g, h) &= {-\tau} \frac{\bar{f}^{\ast}(-\frac{1}{\tau} g(\xvec))}{\mu_0(\xvec)} + h(\yvec) \frac{\nu(\yvec)}{\nu_0(\yvec)} \\
&\quad\quad - \frac{\gamma}{\mu_0(\xvec) \nu_0(\yvec)} \bar{R}^{\ast}\lr{\max\left\{\frac{1}{\gamma} (g(\xvec) + h(\yvec) - d^2(\xvec, \yvec)), \nabla R(\zero)(\xvec, \yvec)\right\}}.
\end{split}
\end{align}

Here, the terms $\bar{f}^{\ast}$ and $\bar{R}^{\ast}$ arise when we express the conjugate functionals $f^{\ast}$ and $R^{\ast}$ in integral form,
\begin{align*}
f^{\ast}(z) &= \int_{\Xspace} \bar{f}^{\ast}(z(\xvec)), \quad\quad R^{\ast}(\xi) = \int_{\Xspace \times \Xspace} \bar{R}^{\ast}(\xi(\xvec, \yvec)).
\end{align*}
In the case of an advection-diffusion, for example, the former is
\begin{align*}
\bar{f}^{\ast}(z(\xvec)) &= \beta^{-1} \exp\lr{\beta (z(\xvec) - w(\xvec))}
\end{align*}
for $w : \Xspace \rightarrow [0, +\infty)$ the advection potential.

\subsection{Monte Carlo approximation}

The stochastic programming formulation \eqref{eq:wassflow-continuous-expectation-max-dual} suggests a Monte Carlo approximation. If we sample $N$ pairs $(\xvec^{(i)}, \yvec^{(i)}) \in \Xspace \times \Xspace$ independently according to $\mu_0 \otimes \nu_0$, we can approximate $\dual_{\nu}^{\gamma, \tau}$ by the empirical mean,
\begin{equation}
\label{eq:wassflow-continuous-dual-empirical-mean-objective}
\dual_{\nu, N}^{\gamma, \tau}(g, h) = \frac{1}{N} \sum_{i=1}^N \dualintegrand_{\nu}^{\gamma, \tau}(\xvec^{(i)}, \yvec^{(i)}, g, h).
\end{equation}
This converges to $\dual_{\nu}^{\gamma, \tau}(g, h)$ in the limit of large $N$. 

The measure $\mu_0 \otimes \nu_0$ functions similarly to the importance distribution in importance sampling. Here, we expect a low variance approximation requires $\mu_0 \otimes \nu_0$ to be similar to $\mu_{\ast} \otimes \nu$, with $\mu_{\ast}$ the exact primal solution for the gradient step. In practice, it suffices to choose a hypercube containing the effective support of $\mu_{\ast} \otimes \nu$, and sample uniformly. This effective support can be determined by a Gaussian approximation to the process, such as underlies the extended or unscented Kalman filter.

\subsection{RKHS approximation}
\label{sec:wassflow-continuous-approx-representation}

There is one more step to obtain a tractable problem: we need to restrict the domain of the dual, to ensure a finite-dimensional solution. We choose a domain $\mathcal{G} \times \mathcal{G}$, with $\mathcal{G}$ a compact, convex subset of a reproducing kernel Hilbert space (RKHS) $\Hilb$ defined on $\Xspace$. From a practical standpoint, this encompasses two settings: the first is the case in which we choose a finite set of basis functions $\{\phi_k\}_{k=1}^p \subset L^2(\Xspace)$ and let $\mathcal{G}$ be contained in their linear span; the second is the case in which we choose a reproducing kernel $\kernel : \Xspace \times \Xspace \rightarrow \reals$ associated to an RKHS $\Hilb$ and assume $\mathcal{G} \subset \Hilb$. In the second case, the fact of a finite-dimensional representation arises from a representer theorem (Proposition \ref{prop:wassflow-continuous-rkhs-representation}). In either case we assume the coefficients are restricted to a compact, convex set.

\begin{proposition}[Representation for general RKHS]
\label{prop:wassflow-continuous-rkhs-representation}
Let $\nu \in \Pspace(\Xspace)$ and $\gamma, \tau, N > 0$. Let $\{(\xvec^{(i)}, \yvec^{(i)}\}_{i=1}^N \subset \Xspace \times \Xspace$. Then there exist $g_{\ast}, h_{\ast} \in \Hilb$ maximizing \eqref{eq:wassflow-continuous-dual-empirical-mean-objective} such that
$$ (g_{\ast}, h_{\ast}) = \sum_{i=1}^N \lr{\alpha_g^{(i)} \kernel(\xvec^{(i)}, \cdot), \alpha_h^{(i)} \kernel(\yvec^{(i)}, \cdot)}, $$
for some sequences of scalar coefficients $\{\alpha_g^{(i)}\}_{i=1}^N$ and $\{\alpha_h^{(i)}\}_{i=1}^N$, with $\kernel : \Xspace \times \Xspace \rightarrow \reals$ the reproducing kernel for $\Hilb$.
\end{proposition}

\subsection{Optimization}

The Monte Carlo stochastic program can be solved by a standard iterative methods for convex optimization. Algorithm \ref{alg:wassflow-continuous-stochastic-approx} outlines the resulting inference method. Note that conditioning of the problem depends on the regularization parameter $\gamma$, which presents a tradeoff between accuracy of the Wasserstein approximation (smaller $\gamma$) and fast optimization (larger).

\begin{algorithm}[t]
\caption{Stochastic program approximating Wasserstein gradient flow}
\begin{algorithmic}
\label{alg:wassflow-continuous-stochastic-approx}
\STATE {\bf Given}: initial density $\rho_t$, constant $\gamma > 0$, timestep $\tau > 0$.
\STATE {\bf Choose} sampling densities $\mu_0, \nu_0$ on $\Xspace$.
\STATE {\bf Sample} independently $N$ pairs $(\xvec_i, \yvec_i) \sim \mu_0 \otimes \nu_0.$
\STATE {\bf Solve} $g_{\ast}, h_{\ast} = \argmax_{g, h \in \mathcal{G}} \dual_{\rho_t, N}^{\gamma, \tau}(g, h).$
\STATE The evolved density is $\rho_{t + \tau} = \nabla f^{\ast}\lr{-\frac{1}{\tau} g_{\ast}}$.
\end{algorithmic}
\end{algorithm}

\section{Properties}

\subsection{Consistency}

The Monte Carlo stochastic program \eqref{eq:wassflow-continuous-dual-empirical-mean-objective} yields a consistent approximation to the regularized Wasserstein gradient step \eqref{eq:wassflow-inference-proximal-operator-smoothed}, in the sense that, as we increase the number of samples, the solution converges to that of the original dual program \eqref{eq:wassflow-continuous-expectation-max-dual}. This holds under a set of assumptions including compactness of $\Xspace \times \Xspace$ and conditions on $\mu_0, \nu_0$ and $\mathcal{G}$ (Appendix \ref{sec:appendix-consistency}). The assumptions guarantee that the stochastic dual objective \eqref{eq:wassflow-continuous-dual-empirical-mean-objective} is $L$-Lipschitz. Under the assumptions, we get uniform convergence of the Monte Carlo dual objective \eqref{eq:wassflow-continuous-dual-empirical-mean-objective} to its expectation \eqref{eq:wassflow-continuous-expectation-max-dual}, and this suffices to guarantee consistency.

\begin{proposition}[Consistency of stochastic program]
\label{prop:wassflow-consistency-unregularized}
Let $\dual_{\nu}^{\gamma, \tau}$ and $\dual_{\nu, N}^{\gamma, \tau}$ be defined as in \eqref{eq:wassflow-continuous-expectation-max-dual} and \eqref{eq:wassflow-continuous-dual-empirical-mean-objective}, respectively, with $\gamma, \tau, N > 0$, and suppose Assumptions {\bf A1}-{\bf A6} hold. Let $(g_N, h_N)$ optimize $\dual_{\nu, N}$ and $(g_{\infty}, h_{\infty})$ optimize $\dual_{\nu}^{\gamma, \tau}$. Then for any $\delta > 0$, with probability at least $1 - \delta$ over the sample of size $N$,
\begin{equation}
\dual_{\nu}^{\gamma, \tau}(g_{\infty}, h_{\infty}) - \dual_{\nu}^{\gamma, \tau}(g_N, h_N) \leq \Ord\lr{\sqrt{\frac{(H K L)^2 \log(1 / \delta)}{N}}}.
\end{equation}
\end{proposition}

\subsection{Computational complexity}

Complexity of first order descent methods for the stochastic dual problem is dominated by evaluation of the functions $g$ and $h$ at each iteration, for each sample $(\xvec_i, \yvec_i)_{i=1}^N$. Each pointwise evaluation of $g$ at a point $\xvec$ (and analogously for $h$ at $\yvec$) requires evaluating the sum $\sum_{k=1}^p \phi_k(\xvec) \alpha_k$, with $\alpha_k$ being the coefficients parameterizing the function \footnote{In the case of a kernel parameterization, we have $p = N$ and $\phi_k(\xvec) = \kernel(\xvec, \xvec^{(k)})$.}. Hence straightforward serial evaluation of $g$ and $h$ at each iteration is $\Ord(N p)$, with $p$ the dimension of $\mathcal{G}$. These sums, however, are trivially parallelizable. Moreover, for certain kernels (notably Gaussian kernels), the serial complexity can be reduced to $\Ord(N)$, by applying a fast multipole method such as the fast Gauss transform \cite{Greengard:1991fast}.

\vspace{-0.08cm}
\section{Empirical performance}
\vspace{-0.01cm}

\subsection{Discussion}

We note that accuracy of the proposed method can vary significantly, depending on several factors, including the particular density being approximated. Even given an unlimited number of Monte Carlo samples, our method gives a biased approximation of the exact diffusion process. There are three sources of bias. First is the discrepancy between the exact Wasserstein gradient step and the exact diffusion process, which only vanishes when the timestep is taken to zero. The second is the regularization applied to the Wasserstein distance, which can move the solution away from the exact Wasserstein gradient step. And the third source is the space $\mathcal{G}$ within which we optimize the dual variables $g$ and $h$, which may not contain the true solution. All three present tradeoffs in accuracy vs. computational complexity of optimization, and represent design choices when applying the method.

\subsection{Performance in high dimensions: Ornstein-Uhlenbeck process}

\begin{figure}[t]
\centering
\begin{subfigure}[t]{0.49\textwidth}
  \centering
  \includegraphics[width=\textwidth]{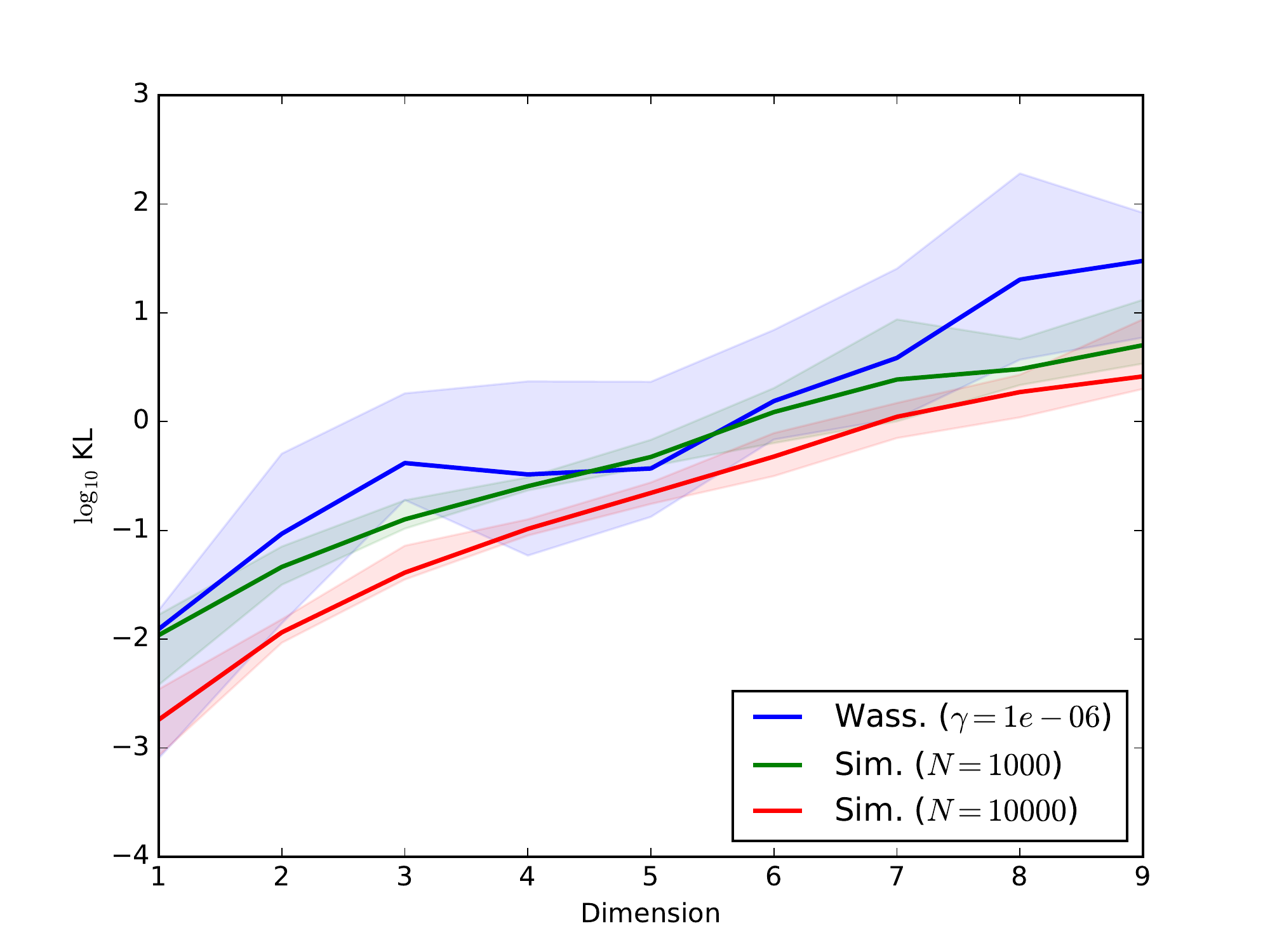}
  \caption{Accuracy with increasing dimension of the domain.}
  \label{fig:wassflow-empirical-exact-vs-dimension}
\end{subfigure}
\begin{subfigure}[t]{0.49\textwidth}
  \centering
  \includegraphics[width=\textwidth]{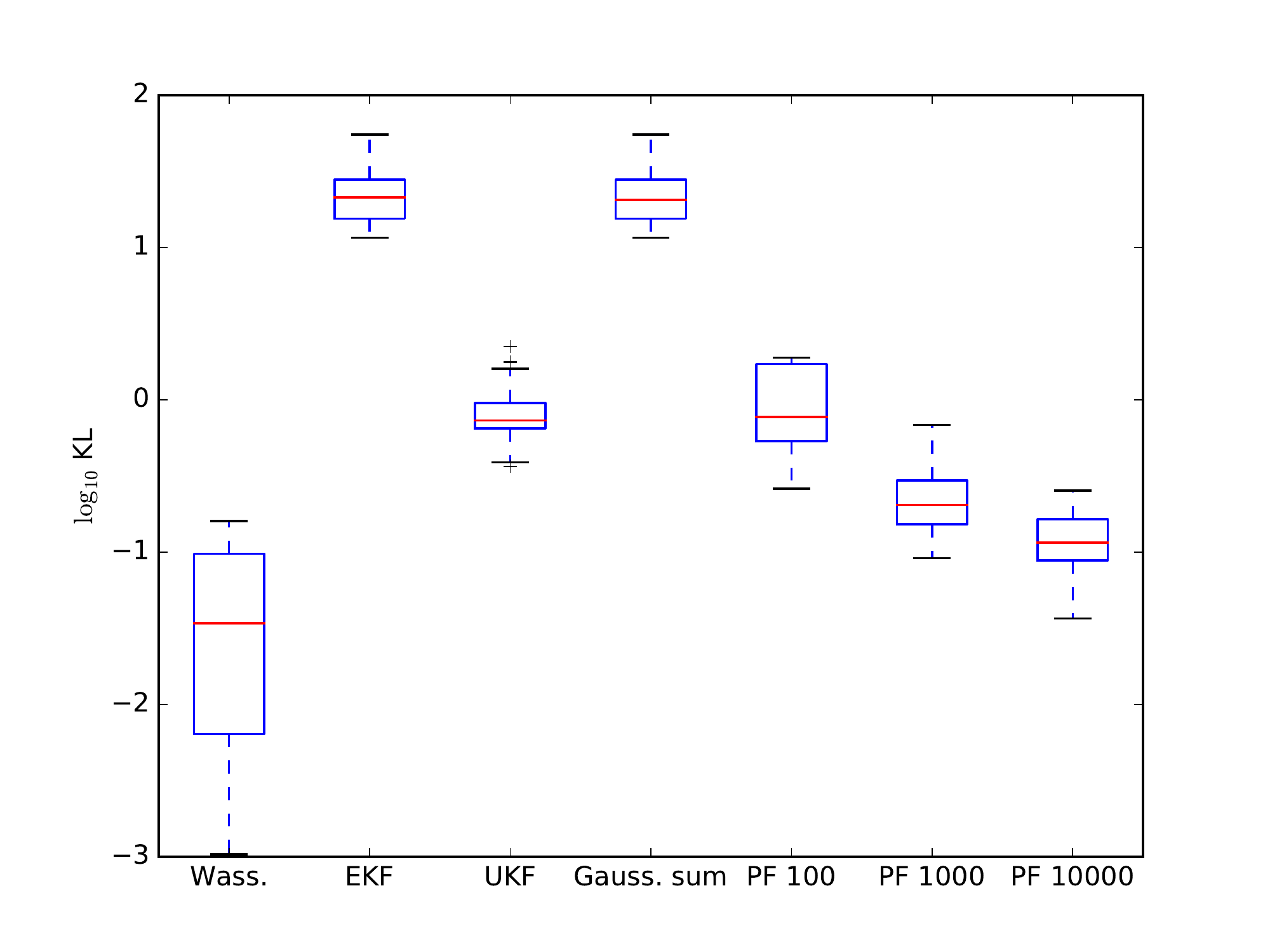}
  \caption{Posterior accuracy in nonlinear filtering.}
  \label{fig:wassflow-empirical-sine-posterior-kl}
\end{subfigure}
\caption{Empirical performance.}
\label{fig:wassflow-empirical}
\vspace{-0.5cm}
\end{figure}

We study the accuracy of our proposed inference method as the dimension of the domain increases. As we have sidestepped the need for discretization of the domain, our approximation is at least computable in arbitrary dimensions. The question is how the accuracy degrades with the dimension.

As a target, we use the only diffusion process of the form \eqref{eq:wassflow-bg-advection-diffusion-kolmogorov-forward} known to have a computable closed form solution in high dimensions. This is the Ornstein-Uhlenbeck process, which is a diffusion with a quadratic potential $w(\xvec) = (\xvec - \bvec)^{\trans} \Amat (\xvec - \bvec)$, parameterized by matrix $\Amat \in \reals^{d \times d}$ and offset $\bvec \in \reals^d$.
Given a deterministic initial condition, the exact solution at time $t$ is Gaussian with mean and covariance evolving in time towards their long-time stationary values. We fix $\beta = 1$ and generate random forcing matrices $\Amat$ and offsets $\bvec$.

As a baseline for comparison, we use the only other approach for high-dimensional inference that doesn't rely on a parametric assumption. This is a standard particle simulation method \footnote{We use the Euler-Maruyama method for simulation, with timestep $10^{-3}$.}, coupled with Gaussian kernel density estimation to obtain the full inferred distribution.

Figure \ref{fig:wassflow-empirical-exact-vs-dimension} shows the accuracy of the two methods as we increase the dimension of the underlying domain $\Xspace$ \footnote{We use an L2 regularizer and set $\gamma = 10^{-6}$. We use a third degree polynomial kernel for approximating $g$ and $h$ and approximate the objective using $2 \cdot 10^4$ sample points. We use a timestep of $\tau = 1/5$.}, for a timestep of $\Delta t = 1$. The figure shows median and $95\%$ interval over $20$ replicates. We see that our method scales with the dimension roughly equivalently to the simulation method, achieving accuracy (in symmetric KL divergence) comparable to simulation with $1000$ particles.

\subsection{Application: nonlinear filtering}

We demonstrate filtering of a nonlinear diffusion, which is observed at discrete times via a noisy measurement process. This is a discrete-time stochastic process $\yrv_k$, taking values at times $t_k$, which is related to the underlying diffusion $\xrv_t$ by 
$$ \yrv_k = \xrv_{t_k} + \vvec_k $$
with $\vvec_k \sim \normal(0, \sigma_{\yrv}^2)$ noise. Given a sequence of such measurements $\yvec_{0:K}$ up to time $t_K$, the {\bf continuous-discrete filtering} problem is that of determining the corresponding distribution over the underlying state, $\Pr(\xrv_t = \xvec_t|\yvec_{0:K})$, at some future time $t \geq t_K$. For future times $t > t_K$, this is the {\bf marginal prior} or {\bf predictive distribution} over states, defined by the dynamics of the diffusion process, satisfying the forward PDE \eqref{eq:wassflow-bg-diffusive-PDE} with initial density $\Pr(\xrv_{t_K} = \xvec_{t_K}|\yvec_{0:K})$.
At the measurement time $t = t_K$, this is the {\bf marginal posterior}, conditional upon the measurements, and is defined by a recursive update equation
\begin{equation*}
\label{eq:wassflow-bg-filtering-recursive-update}
\Pr(\xrv_{t_K} = \xvec_{t_K}|\yvec_{0:K}) = \frac{\Pr(\yrv_K = \yvec_K|\xrv_{t_K} = \xvec_{t_K}) \Pr(\xrv_{t_K} = \xvec_{t_K}|\yvec_{0:K-1})}{\Pr(\yrv_K = \yvec_K)}.
\end{equation*}
The term $\Pr(\xrv_{t_K} = \xvec_{t_K}|\yvec_{0:K-1})$ is the predictive distribution given the measurements up to time $t_{K-1}$. We assume an initial distribution $\Pr(\xrv_{t_0} = \xvec_{t_0})$ is given.

We assume the underlying state evolves according to a diffusion in the potential $w(x) = \frac{1}{\pi} \sin(2 \pi x) + \frac{1}{4} x^2$, having unit diffusion coefficient $\beta = 1$. This is a highly nonlinear process and yields multimodal posteriors, which will present a challenge for most existing filtering methods. Measurements are made with noise $\sigma = 1$. We apply the Wasserstein gradient flow to approximate the predictive density of the diffusion, which at measurement times is multiplied pointwise with the likelihood $\Pr(\yvec_k|\xvec_{t_k})$ to obtain an unnormalized posterior density \footnote{We use an L2 regularizer and set $\gamma = 10^{-6}$. We use a Gaussian kernel with bandwidth $0.1$ and approximate the objective with $10^4$ samples. We use a timestep of $\tau = 1 / 4$.}.

We use five methods as baselines for comparison. The first computes the exact predictive density by numerically integrating the Fokker-Planck equation \eqref{eq:wassflow-bg-advection-diffusion-kolmogorov-forward} on a fine grid -- this allows us to compare computed posteriors to the exact, true posterior. The second and third are the Extended and Unscented Kalman filters, which maintain Gaussian approximations to the posterior. The fourth method is a Gaussian sum filter \cite{Alspach:1972to}, which approximates the posterior by a mixture of Gaussians. And the fifth baseline is a bootstrap particle filter, which samples particles according to the transition density $\Pr(\xvec_{t_k}|\xvec_{t_{k-1}})$, by numerical forward simulation of the SDE \eqref{eq:wassflow-bg-advection-diffusion-sde} \footnote{For foward simulation, we use an Euler's method with timestep $10^{-3}$.}.

We simulate $20$ observations at a time interval of $\Delta t = 1$, and compute the posterior density by each of the methods. Figure \ref{fig:wassflow-empirical-sine-posterior-kl} shows quantitatively the fidelity of the estimated posterior to that computed by exact numerical integration, repeating the filtering experiment $100$ times. Appendix \ref{sec:appendix-filtering} shows examples of the estimated posterior density of the diffusion. The Wasserstein gradient flow consistently outperforms the baselines, both qualitatively and quantitatively, achieving smaller symmetric KL divergence to the true posterior. Whereas the multimodality of the posterior presents a challenge for the baseline methods, the Wasserstein gradient flow captures it almost exactly.

\section*{Acknowledgments}

The authors acknowledge the generous support of the Shell/MIT Energy Initiative, and thank Justin Solomon for some very helpful discussions.

\bibliographystyle{unsrt}
{\small
\bibliography{jko}
}

\newpage
\appendix

\section{Duality}

\begin{repproposition}{prop:wassflow-inference-fenchel-dual}[Strong duality]
Let $\nu \in \Pspace(\Xspace)$ and $f : \Pspace(\Xspace) \rightarrow [0, +\infty)$ a convex, lower semicontinuous and proper functional. Define $\primal_{\nu}^{\gamma, \tau}$ as in \eqref{eq:wassflow-inference-primal-objective} and $\dual_{\nu}^{\gamma, \tau}$ as in \eqref{eq:wassflow-inference-dual-objective}. Assume $\gamma > 0$. Then
\begin{equation}
\min_{\mu \in \Pspace(\Xspace)} \primal_{\nu}^{\gamma, \tau}(\mu) = \max_{g \in L^2(\Xspace), h \in L^2(\Xspace)} \dual_{\nu}^{\gamma, \tau}(g, h).
\end{equation}
Suppose $f$ is strictly convex and let $g_{\ast}, h_{\ast}$ maximize $\dual_{\nu}^{\gamma, \tau}$. Then
\begin{equation}
\mu_{\ast} = \nabla f^{\ast}(-\frac{1}{\tau} g_{\ast})
\end{equation}
minimizes $\primal_{\nu}^{\gamma, \tau}$.
\end{repproposition}

\begin{proof}
For $\Wass_{\gamma}(\cdot, \nu)$ and $f$ both convex, lower semicontinuous and proper, Fenchel duality has that
\begin{equation}
\label{eq:wassflow-smoothed-dual-fenchel-duality}
\min_{\mu \in L^2(\Xspace)} \Wass_{\gamma}(\mu, \nu) + \tau f(\mu) = \max_{g \in L^2(\Xspace)} -\Wass_{\gamma}(\cdot, \nu)^{\ast}(g) - \tau f^{\ast}(-\frac{1}{\tau} g),
\end{equation}
with $\Wass_{\gamma}(\cdot, \nu)^{\ast}$ and $f^{\ast}$ the convex conjugates,
\begin{align}
\label{eq:wassflow-smoothed-dual-fenchel-convex-conjugate-wass}
\Wass_{\gamma}(\cdot, \nu)^{\ast}(g) &= \max_{\mu \in \Pspace(\Xspace)} \langle d\mu, g \rangle_{L^2(\Xspace)} - \Wass_{\gamma}(\mu, \nu), \\
(\tau f)^{\ast}(-g) = \tau f^{\ast}(-\frac{1}{\tau} g) &= \max_{\mu \in \Pspace(\Xspace)} -\langle d\mu, g \rangle_{L^2(\Xspace)} - \tau f(\mu).
\end{align}

Rewrite $\Wass_{\gamma}(\cdot, \nu)^{\ast}$.
\begin{align*}
\Wass_{\gamma}(\cdot, \nu)^{\ast}(g) = -\inf_{\pi \in \Pspace(\Xspace \times \Xspace), \Proj_2 \pi = \nu} \langle c - g, d\pi \rangle_{L^2(\Xspace \times \Xspace)} + \gamma R(\pi).
\end{align*}

The Lagrangian dual for $\Wass_{\gamma}(\cdot, \nu)^{\ast}$ is
\begin{align*}
\Wass_{\gamma}(\cdot, \nu)^{\ast}(g) &= -\sup_{h \in L^2(\Xspace), \varepsilon \in L^2(\Xspace \times \Xspace)} \inf_{\pi \in \mathcal{M}_+(\Xspace \times \Xspace)} \langle c - g, d\pi \rangle_{L^2(\Xspace \times \Xspace)} + \gamma R(\pi) \\
 &\quad\quad - \langle \varepsilon, d\pi \rangle_{L^2(\Xspace \times \Xspace)} + \langle h, d(\nu - \Proj_2 \pi) \rangle_{L^2(\Xspace)}.
\end{align*}
From the KKT conditions, we get
\begin{align*}
& c - g + \gamma \nabla R(\pi) - \varepsilon - h = 0 \\
& \varepsilon, \pi \geq 0 \\
& \varepsilon \pi = 0.
\end{align*}
The first condition implies
\begin{align*}
& \nabla R(\pi) = \frac{1}{\gamma} \lr{g + h - c + \varepsilon} \\
\Rightarrow & d\pi = \nabla R^{\ast}\lr{\frac{1}{\gamma} \lr{g + h - c + \varepsilon}}, \\
\end{align*}
because we assumed $R$ is Legendre, so its gradient map is a bijection between $\interior \dom R$ and $\interior \dom R^{\ast}$ having inverse $\nabla R^{\ast}$.

Suppose $\nabla R^{\ast}\lr{\frac{1}{\gamma} \lr{g + h - c}}(\xvec, \yvec) < 0$. As $R$ is separable, we have $\nabla R^{\ast}(\xi)(\xvec, \yvec) = \nabla \bar{R}^{\ast}(d\xi(\xvec, \yvec))$, for $\bar{R}^{\ast}$ the pointwise component function for $R^{\ast} = \int_{\Xspace \times \Xspace} \bar{R}^{\ast}(\xvec, \yvec)$. And the gradient map is monotonic, so there exists positive $\Delta \in \reals_{++}$ such that
\begin{align*}
\nabla \bar{R}^{\ast}\lr{\frac{1}{\gamma} \lr{g(\xvec) + h(\yvec) - c(\xvec, \yvec) + \Delta}} = 0.
\end{align*}
Choosing $\varepsilon(\xvec, \yvec) = \Delta$ then yields $d\pi(\xvec, \yvec) = \nabla R^{\ast}\lr{\frac{1}{\gamma} \lr{g + h - c + \varepsilon}}(\xvec, \yvec) = 0$, and $\varepsilon$ defined this way is feasible ($\varepsilon(\xvec, \yvec)$ and $d\pi(\xvec, \yvec)$ are nonnegative and satisfy complementary slackness). Moreover, any other choice of $\varepsilon$ yields either $d\pi(\xvec, \yvec) < 0$, violating nonnegativity, or $d\pi(\xvec, \yvec) > 0$,  violating complementary slackness, as $\nabla \bar{R}^{\ast}$ is injective. Hence, $\varepsilon(\xvec, \yvec)$ is necessarily set to $\Delta$ and we have that $d\pi(\xvec, \yvec) = 0$. Clearly $d\pi(\xvec, \yvec) > 0$ implies $\varepsilon(\xvec, \yvec) = 0$, so we have that
\begin{equation}
d\pi = \lr{\nabla R^{\ast}\lr{\frac{1}{\gamma}(g + h - c)}}_+,
\end{equation}
with $(u)_+(\xvec, \yvec) = \max\left\{u(\xvec, \yvec), 0\right\}$ for any $u : \Xspace \times \Xspace \rightarrow \reals$.

Equivalently, we can write
\begin{align*}
\varepsilon(\xvec, \yvec) = \bigbrace{0 & \frac{1}{\gamma}(g(\xvec) + h(\yvec) - c(\xvec, \yvec)) > \nabla \bar{R}(0) \\ \gamma \nabla \bar{R}(0) - (g(\xvec) + h(\yvec) - c(\xvec, \yvec)) & \frac{1}{\gamma}(g(\xvec) + h(\yvec) - c(\xvec, \yvec)) \leq \nabla \bar{R}(0)}.
\end{align*}
Hence, optimal joint measure $\pi$ equivalently satisfies
\begin{equation}
d\pi = \nabla R^{\ast}\lr{\max \left\{\frac{1}{\gamma}(g + h - c), \nabla R(0)\right\}}.
\end{equation}

By definition of the convex conjugate,
\begin{align*}
R\lr{\nabla R^{\ast}(\xi)} = \langle \xi, \nabla R^{\ast}(\xi) \rangle_{L^2(\Xspace \times \Xspace)} - R^{\ast}(\xi),
\end{align*}
so plugging optimal $\pi$ into the Lagrangian dual for $\Wass_{\gamma}(\cdot, \nu)^{\ast}$, we get
\begin{equation}
\Wass_{\gamma}(\cdot, \nu)^{\ast}(g) = - \sup_{h \in L^2(\Xspace)} \langle h, \nu \rangle - \gamma R^{\ast}\lr{\max \left\{\frac{1}{\gamma}(g + h - c), \nabla R(0)\right\}}.
\end{equation}
From \eqref{eq:wassflow-smoothed-dual-fenchel-duality}, then we get the Fenchel dual
\begin{equation}
\dual_{\nu}^{\gamma, \tau}(g, h) = -\tau f^{\ast}\lr{-\frac{1}{\tau} g} + \langle h, \nu \rangle - \gamma R^{\ast}\lr{\max\left\{\frac{1}{\gamma}(g + h - c), \nabla R(0)\right\}}.
\end{equation}

Suppose $g_{\ast}, h_{\ast} \in L^2(\Xspace)$ optimize the dual objective $\dual_{\nu}^{\gamma, \tau}$. Then $\mu_{\ast}$ optimal for $\primal_{\nu}^{\gamma, \tau}$ satisfies
$$ \mu_{\ast} \in \partial (\tau f)^{\ast}(-g_{\ast}). $$
When $f$ is strictly convex, this is $\mu_{\ast} = \nabla (\tau f)^{\ast}(-g_{\ast}) = \nabla f^{\ast}(-\frac{1}{\tau} g_{\ast})$.
\end{proof}

\section{Representer theorem}

\begin{repproposition}{prop:wassflow-continuous-rkhs-representation}[Representation for general RKHS]
Let $\nu \in \Pspace(\Xspace)$ and $\gamma, \tau, N > 0$. Let $\{(\xvec^{(i)}, \yvec^{(i)}\}_{i=1}^N \subset \Xspace \times \Xspace$. Then there exist $g_{\ast}, h_{\ast} \in \Hilb$ maximizing \eqref{eq:wassflow-continuous-dual-empirical-mean-objective-regularized} such that
$$ (g_{\ast}, h_{\ast}) = \sum_{i=1}^N \lr{\alpha_g^{(i)} \kernel(\xvec^{(i)}, \cdot), \alpha_h^{(i)} \kernel(\yvec^{(i)}, \cdot)}, $$
for some sequences of scalar coefficients $\{\alpha_g^{(i)}\}_{i=1}^N$ and $\{\alpha_h^{(i)}\}_{i=1}^N$, with $\kernel : \Xspace \times \Xspace \rightarrow \reals$ the reproducing kernel for $\Hilb$.
\end{repproposition}

\begin{proof}
Let $\Hilb$ be the RKHS having kernel $\kernel$, and let $\langle \cdot, \cdot \rangle_{\Hilb} : \Hilb \times \Hilb \rightarrow \reals$ be the associated inner product. Let $g \in \Hilb$. From the reproducing property of $\Hilb$, we have that pointwise evaluation is a linear functional such that $g(\xvec) = \langle g, \kernel(\xvec, \cdot) \rangle_{\Hilb}$, for all $\xvec \in \Xspace$.

Let $\Hilb_N \subset \Hilb$ be the linear span of the functions $\kernel(\xvec^{(i)}, \cdot)$, and $\Hilb_N^{\perp}$ its orthogonal complement. For any $g \in \Hilb$, we can decompose it as $g = g^{\parallel} + g^{\perp}$, with $g^{\parallel} \in \Hilb_N$ and $g^{\perp} \in \Hilb_N^{\perp}$. Moreover, $\dual_{\nu, N}^{\gamma, \tau}(g, h) = \dual_{\nu, N}^{\gamma, \tau}(g^{\parallel}, h)$, as $\dual_{\nu, N}^{\gamma, \tau}$ depends on its first argument only via the evaluation functional at each point,
$$ g(\xvec^{(i)}) = \langle \kernel(\xvec^{(i)}, \cdot), g \rangle_{\Hilb} = \langle \kernel(\xvec^{(i)}, \cdot), g^{\parallel} \rangle_{\Hilb}. $$
Hence if $\dual_{\nu, N}^{\gamma, \tau}$ is maximized by $g_{\ast}$, it is also maximized by $g_{\ast}^{\parallel} \in \Hilb_N$. The same argument holds for $h_{\ast}$.
\end{proof}

\section{Consistency}
\label{sec:appendix-consistency}

We make the following assumptions.
\begin{itemize}
\item[{\bf A1}] $\Xspace \times \Xspace$ is compact.
\item[{\bf A2}] $\mu_0$ and $\nu_0$ are bounded away from zero: $\mu_0(\xvec) \geq U_0^{\min} > 0$, $\nu_0(\yvec) \geq V_0^{\min} > 0$, for all $\xvec, \yvec \in \Xspace$.
\item[{\bf A3}] $\mathcal{G}$ is compact and convex, with $\|g\|_{\Hilb} \leq H$ for all $g \in \mathcal{G}$.
\item[{\bf A4}] $\Hilb$ has reproducing kernel $\kernel$ that is bounded: $\max_{\xvec \in \Xspace} \sqrt{\kernel(\xvec, \xvec)} = K < \infty$. 
\item[{\bf A5}] $\bar{f}^{\ast}$ is convex and $L_{f^{\ast}}$-Lipschitz.
\item[{\bf A6}] $\dom \bar{R}^{\ast} = \reals$.
\end{itemize}

The assumptions guarantee that the Monte Carlo dual objective \eqref{eq:wassflow-continuous-dual-empirical-mean-objective} is $L$-Lipschitz.

\begin{proposition}[Lipschitz property for $\dualintegrand_{\nu}^{\gamma, \tau}$]
\label{prop:appendix-lipschitz}
Let $\dualintegrand_{\nu}^{\gamma, \tau}$ be defined as in \eqref{eq:wassflow-continuous-expectation-max-dual-integrand} and suppose Assumptions {\bf A1}-{\bf A6} hold. Let $U^{\max} = \max_{\xvec \in \Xspace, g \in \Hilb} \frac{\nabla f^{\ast}(-\frac{1}{\tau} g(\xvec))}{\mu_0(\xvec)}$ and $V^{\max} = \max_{\yvec \in \Xspace} \frac{\nu(\yvec)}{\nu_0(\yvec)}$. Then for all $g, g^{\prime}, h, h^{\prime} \in \Hilb$, $\dualintegrand_{\nu}^{\gamma, \tau}$ satisfies
\begin{align*}
& \left| \dualintegrand_{\nu}^{\gamma, \tau}(\xvec, \yvec, g, h) - \dualintegrand_{\nu}^{\gamma, \tau}(\xvec, \yvec, g^{\prime}, h^{\prime}) \right| \leq L \| (g(\xvec), h(\yvec)) - (g^{\prime}(\xvec), h^{\prime}(\yvec)) \|_1
\end{align*}
with constant $L$ defined by
$L = \max\left\{U^{\max}, V^{\max}, \frac{\nabla \bar{R}^{\ast}\lr{\frac{2}{\gamma} K H}}{U_0^{\min} V_0^{\min}}\right\}$.
\end{proposition}

\begin{proof}
Note that $U^{\max}$ and $V^{\max}$ are finite by assumptions {\bf A2} and {\bf A5}. 

By {\bf A3-A4}, we have that $K = \min_{\xvec \in \Xspace} \sqrt{\kernel(\xvec, \xvec)} < \infty$, and $\mathcal{G} \times \mathcal{G}$ is bounded, such that $\|g\|_{\Hilb}, \|h\|_{\Hilb} \leq H$. Therefore $|g(\xvec)|, |h(\yvec)| \leq K H$, because by the reproducing property
\begin{align*}
|g(\xvec)| &= |\langle \kernel(\xvec, \cdot), g \rangle_{\Hilb}| \\
 &\leq \| \kernel(\xvec, \cdot) \|_{\Hilb} \| g \|_{\Hilb} \\
 &\leq K \|g\|_{\Hilb}, \\
 &\leq K H,
\end{align*} 
with the second step from Cauchy-Schwarz. The analogous result holds for $|h(\yvec)|$.

Let $q(g(\xvec), h(\yvec)) = \frac{1}{\gamma}(g(\xvec) + h(\yvec) - d^2(\xvec, \yvec))$. Then $\dualintegrand_{\nu}^{\gamma, \tau}$ has subderivatives
$$ \frac{\partial \dualintegrand_{\nu}^{\gamma, \tau}}{\partial g(\xvec)} = \frac{\nabla \bar{f}^{\ast}(-\frac{1}{\tau} g(\xvec))}{\mu_0(\xvec)} - \frac{\gamma}{\mu_0(\xvec) \nu_0(\yvec)} \bigbrace{\frac{1}{\gamma} \nabla \bar{R}^{\ast}\lr{q(g(\xvec), h(\yvec))} & q(g(\xvec), h(\yvec)) > \nabla \bar{R}(0) \\ {[0, \frac{1}{\gamma} \nabla \bar{R}^{\ast}(q(g(\xvec), h(\yvec)))]} & q(g(\xvec), h(\yvec)) = \nabla \bar{R}(0) \\ 0 & \otherwise} $$
in $g(\xvec)$ and 
$$ \frac{\partial \dualintegrand_{\nu}^{\gamma, \tau}}{\partial h(\yvec)} = \frac{\nu(\yvec)}{\nu_0(\yvec)} - \frac{\gamma}{\mu_0(\xvec) \nu_0(\yvec)} \bigbrace{\frac{1}{\gamma} \nabla \bar{R}^{\ast}\lr{q(g(\xvec), h(\yvec))} & q(g(\xvec), h(\yvec)) > \nabla \bar{R}(0) \\ {[0, \frac{1}{\gamma} \nabla \bar{R}^{\ast}(q(g(\xvec), h(\yvec)))]} & q(g(\xvec), h(\yvec)) = \nabla \bar{R}(0) \\ 0 & \otherwise} $$
in $h(\yvec)$. In both cases, the second term subtracts a nonnegative quantity while the first term is nonnegative. As $g(\xvec)$ and $h(\yvec)$ are bounded, $q$ is bounded from above, with
$$ q(g(\xvec), h(\yvec)) \leq \frac{2}{\gamma} KH. $$
$\nabla \bar{R}^{\ast}$ is monotonic, so it is bounded above by $\nabla \bar{R}^{\ast}\lr{\frac{2}{\gamma} K H}$. We therefore have
\begin{align*}
\left| \frac{\partial \dualintegrand_{\nu}^{\gamma, \tau}}{\partial g(\xvec)} \right| &\leq \max\left\{U^{\max}, \frac{\nabla \bar{R}^{\ast}\lr{\frac{2}{\gamma} K H}}{U_0^{\min} V_0^{\min}}\right\} \triangleq L_g \\
\left| \frac{\partial \dualintegrand_{\nu}^{\gamma, \tau}}{\partial h(\yvec)} \right| &\leq \max\left\{V^{\max}, \frac{\nabla \bar{R}^{\ast}\lr{\frac{2}{\gamma} K H}}{U_0^{\min} V_0^{\min}}\right\} \triangleq L_h.
\end{align*}
$\bar{R}^{\ast}$ is smooth on $\interior \dom \bar{R}^{\ast}$, and $\frac{2}{\gamma} K H \in \interior \dom \bar{R}^{\ast}$ by Assumption {\bf A6}, so $\nabla \bar{R}^{\ast}\lr{\frac{2}{\gamma} K H}$ is finite.

Letting $L = \max\{L_g, L_h\}$, this implies
$$ \left| \dualintegrand_{\nu}^{\gamma, \tau}(\xvec, \yvec, g, h) - \dualintegrand_{\nu}^{\gamma, \tau}(\xvec, \yvec, g^{\prime}, h^{\prime}) \right| \leq L \|(g(\xvec), h(\yvec)) - (g^{\prime}(\xvec), h^{\prime}(\yvec))\|_1, $$
for all $(g, h), (g^{\prime}, h^{\prime}) \in \mathcal{G} \times \mathcal{G}$ and $(\xvec, \yvec) \in \Xspace \times \Xspace$.
\end{proof}

Note that assumption {\bf A5} is satisfied by an advection-diffusion, so long as we assume $w$ is bounded below, as
$$ \max_{g \in \mathcal{G}, \xvec \in \Xspace} \left| \nabla f^{\ast}(-\frac{1}{\tau} g(\xvec)) \right| = \max_{g \in \mathcal{G}, \xvec \in \Xspace} \exp(-\frac{\beta}{\tau} g(\xvec) - w(\xvec)) \leq \exp(\frac{\beta}{\tau} K H - \beta W) $$
with $W = \min_{\xvec \in \Xspace} w(\xvec)$.

Under the assumptions, then, we get uniform convergence of the stochastic dual objective \eqref{eq:wassflow-continuous-dual-empirical-mean-objective} to its expectation \eqref{eq:wassflow-continuous-expectation-max-dual}, and this suffices to guarantee consistency.

\begin{repproposition}{prop:wassflow-consistency-unregularized}[Consistency of stochastic program]
Let $\dual_{\nu}^{\gamma, \tau}$ and $\dual_{\nu, N}^{\gamma, \tau}$ be defined as in \eqref{eq:wassflow-continuous-expectation-max-dual} and \eqref{eq:wassflow-continuous-dual-empirical-mean-objective}, respectively, with $\gamma, \tau, N > 0$, and suppose Assumptions {\bf A1}-{\bf A6} hold. Let $(g_N, h_N)$ optimize $\dual_{\nu, N}$ and $(g_{\infty}, h_{\infty})$ optimize $\dual_{\nu}^{\gamma, \tau}$. Then for any $\delta > 0$, with probability at least $1 - \delta$ over the sample of size $N$,
\begin{equation}
\dual_{\nu}^{\gamma, \tau}(g_{\infty}, h_{\infty}) - \dual_{\nu}^{\gamma, \tau}(g_N, h_N) \leq \Ord\lr{\sqrt{\frac{(H K L)^2 \log(1 / \delta)}{N}}}.
\end{equation}
\end{repproposition}

\begin{proof}
Note that $\dualintegrand_{\nu}^{\gamma, \tau}$ is jointly convex in $g(\xvec)$ and $h(\yvec)$, and these are in linear in $g$ and $h$, respectively. They can be written $g(\xvec) = \langle g, \kernel(\xvec, \cdot) \rangle_{\Hilb}$ with $\|\kernel(\xvec, \cdot)\|_{\Hilb} \leq K$ and $\|g\|_{\Hilb} \leq H$, and similarly for $h(\yvec)$, with the same bounds.

From \cite{ShalevShwartz:2009va} Thm. 1, then, we have uniform convergence of the empirical functional to its expectation, such that with probability $1 - \delta$
$$ \sup_{g, h \in \Hilb} \left|\dual_{\nu}^{\gamma, \tau}(g, h) - \dual_{\nu, N}^{\gamma, \tau}(g, h)\right| \leq \Ord\lr{\sqrt{\frac{(H K L)^2 \log(1 / \delta)}{N}}}, $$
for any $g, h \in \mathcal{G}$. This implies
\begin{align*}
& \dual_{\nu}^{\gamma, \tau}(g_{\infty}, h_{\infty}) - \dual_{\nu, N}^{\gamma, \tau}(g_{\infty}, h_{\infty}) + \dual_{\nu, N}^{\gamma, \tau}(g, h) - \dual_{\nu}^{\gamma, \tau}(g, h) \leq \Ord\lr{\sqrt{\frac{(H K L)^2 \log(1 / \delta)}{N}}} \\
\Rightarrow & \dual_{\nu}^{\gamma, \tau}(g_{\infty}, h_{\infty}) - \dual_{\nu}^{\gamma, \tau}(g, h) \leq \lr{\dual_{\nu, N}^{\gamma, \tau}(g_{\infty}, h_{\infty}) - \dual_{\nu, N}^{\gamma, \tau}(g, h)} + \Ord\lr{\sqrt{\frac{(H K L)^2 \log(1 / \delta)}{N}}} \\
 & \hspace{4cm} \leq \lr{\dual_{\nu, N}^{\gamma, \tau}(g_N, h_N) - \dual_{\nu, N}^{\gamma, \tau}(g, h)} + \Ord\lr{\sqrt{\frac{(H K L)^2 \log(1 / \delta)}{N}}} \\
\end{align*}
for any $g, h \in \mathcal{G}$. In particular, it's true for $g = g_N$ and $h = h_N$, which yields the statement.
\end{proof}

\section{Gradient flow approximates exact diffusion}

In Figure \ref{fig:wassflow-intro-example-exact}, the diffusion is an Ornstein-Uhlenbeck process with potential $w(x) = x^2$ and dispersion $\beta = 1$. The exact solution for the probability density is computed by Chang and Cooper's method on a grid of $400$ points on the interval $[-3, 3]$. The initial condition is a mixture of two Gaussians, centered at $\pm 1$, and each having standard deviation $1$. The Wasserstein gradient flow is computed using a Gaussian kernel supported at $40$ points chosen uniformly at random from $[-3, 3]$, with bandwidth $5 \cdot 10^{-2}$. The objective is approximated with $3 \cdot 10^4$ Monte Carlo samples. We use an entropic regularizer for the Wasserstein distance, with $\gamma = 10^{-2}$, and set timestep $\tau = 1 \cdot 10^{-2}$. The figure shows the density at times $t = 0.05, 0.2, 0.5$.

\section{Accuracy in high dimensions: Ornstein-Uhlenbeck process}

The process we are approximating in Figure \ref{fig:wassflow-empirical-exact-vs-dimension} is an Ornstein-Uhlenbeck process, having potential $w(\xvec) = (\xvec - \bvec) \Amat (\xvec - \bvec)$, with $\Amat \in \reals^{d \times d}$ and $\bvec \in \reals^d$ chosen randomly: $\Amat$ is a diagonal matrix with diagonal elements gamma distributed with shape $2$ and scale $0.5$, while $\bvec$ has independent normally distributed elements with standard deviation $0.5$. The process has dispersion $\beta = 1$, and initial density a delta function at $0$. Its density is computed exactly, in closed form, at time $\Delta t = 1$.

Our baseline is a particle simulation. For each particle, we forward simulate from time $t = 0$, using the Euler-Maruyama method with timestep $10^{-3}$. We use $N = 1000, 10000$ particles.

For the Wasserstein gradient flow, we approximate the objective using $2 \cdot 10^4$ Monte Carlo samples. We use a polynomial kernel of degree three, and an $L^2$ regularizer for the Wasserstein distance, with $\gamma = 10^{-6}$. We set timestep $\tau = 0.2$.

To evaluate the accuracy, we estimate the symmetric KL divergence between the estimated and exact densities by Monte Carlo, sampling $4 \cdot 10^4$ points randomly from the exact solution distribution at $t = 1$. For both estimation methods, we care about the accuracy up to normalization of the estimated distribution. Before computing the divergence, we choose the normalization constant that minimizes the sum of squared errors between the estimated and exact distribution.

We repeat the experiment $20$ times, for $20$ different random potentials, with Figure \ref{fig:wassflow-empirical-exact-vs-dimension} showing the median and $95\%$ interval for each method.

\section{Nonlinear filtering}
\label{sec:appendix-filtering}

\subsection{Problem setup and data generation}

Latent state trajectories in $\reals$ are generated from the SDE model
$$ d\xvec_t = -\lr{2 \cos(2 \pi \xvec_t) + \frac{1}{2} \xvec_t} dt + d\Wiener_t $$
which is an advection-diffusion with potential $w(\xvec) = \frac{1}{\pi} \sin(2\pi \xvec) + \frac{1}{4} \xvec_t^2$ and inverse dispersion coefficient $\beta = 1$. The latent system is observed at a time interval of $\Delta t = 1$, with additive Gaussian noise having standard devation $\sigma = 1$. State trajectories are generated by simulating the SDE using an Euler-Maruyama method with timestep $10^{-3}$, starting from $\xvec_0 = 0$.

\subsection{Baselines}

\noindent
{\bf Discretized numerical integration.} We construct a regularly-spaced grid of $1000$ points on the interval $[-4, 4]$, and use Chang and Cooper's method \cite{Chang:1970wj} to integrate the Fokker-Planck equation for the dynamics. We use a timestep of $10^{-3}$ for the integration.

When filtering, we obtain the posterior state distribution by first propagating forward the posterior at the previous observation time, via integrating the Fokker-Planck equation, then multiplying the resulting distribution pointwise by the observation likelihood and normalizing to sum to one.

\noindent
{\bf Extended Kalman filter.} The extended Kalman filter is implemented as described in \cite{Brown:1997aa}. We use Scipy's {\tt odeint} to integrate the ODE for the mean and covariance. The EKF is initialized with a Gaussian of whose mean is drawn from a normal distribution having mean $0$ and standard deviation $0.1$, and whose variance is $10^{-4}$.

\noindent
{\bf Unscented Kalman filter.} The unscented Kalman filter is implemented as described in \cite{Sarkka:2007td}. We use Scipy's {\tt odeint} to integrate the ODE for the mean and covariance. The UKF is initialized with a Gaussian of mean $0$ and variance $10^{-4}$. We use parameters $\alpha = \frac{1}{2}$, $\beta = 2$, $\kappa = 1$. ($\beta$ here refers to the parameter in \cite{Sarkka:2007td}, rather than the inverse dispersion coefficient in the main text.)

\noindent
{\bf Gaussian sum filter.} We implement a Gaussian sum filter as described in \cite{Alspach:1972to}. The filter is initialized with a mixture of eight Gaussians, having means drawn independently from a normal distribution with mean $0$ and standard deviation $1$, and each having variance $10^{-4}$.

\noindent
{\bf Bootstrap particle filter.} The bootstrap particle filter is implemented as described in \cite{Gordon:1993up}. For propagating particles forward in time, we simulate the system dynamics using an Euler-Maruyama method with timestep $10^{-3}$. We resample trajectories after each observation. To extrapolate the posterior to new points, we use Gaussian kernel density estimation on sampled support points with bandwidth chosen by Scott's rule.

\subsection{Example posterior evolution}

Figure \ref{fig:wassflow-empirical-filter-posterior-evolution-sine} shows an example of the evolution of the posterior distribution for consecutive timesteps. We simulate system trajectories and observations as described above and use the stochastic program for the Wasserstein gradient flow (Section \ref{sec:wassflow-inference-via-stochastic-programming}) to propagate the posterior at one observation time to the next. The resulting distribution is multiplied pointwise by the likelihood to obtain an unnormalized posterior. The sampling distribution for the stochastic program is uniform on the interval $[-4, 4]$. We use an L2 regularizer with $\gamma = 10^{-6}$, a Gaussian kernel with bandwidth $0.1$, and $10^4$ samples for approximating the stochastic program objective. We solve the stochastic program using L-BFGS (from {\tt scipy.optimize}), stopping when the norm of the gradient is less than $10^{-8}$.

We additionally overlay posterior distributions for the baseline algorithms. The distribution obtained from discretized numerical integration is shaded in blue. For visualization, all distributions are sampled on a grid and normalized to sum to one.

\subsection{Quantitative comparison of methods}

We simulate $100$ independent latent state trajectories and their observations. For each we obtain posterior distributions for the proposed Wasserstein gradient flow approximation and the baseline methods, as described above. We sample the resulting distributions on the same grid as was used for discretized numerical integration and normalize to sum to one. We compute the symmetric $\KL$-divergence between the exact distribution from discretized numerical integration and the approximate distribution from the given method.

\begin{figure}[t]
\centering
\begin{subfigure}[t]{0.4\textwidth}
  \centering
  \includegraphics[width=\textwidth]{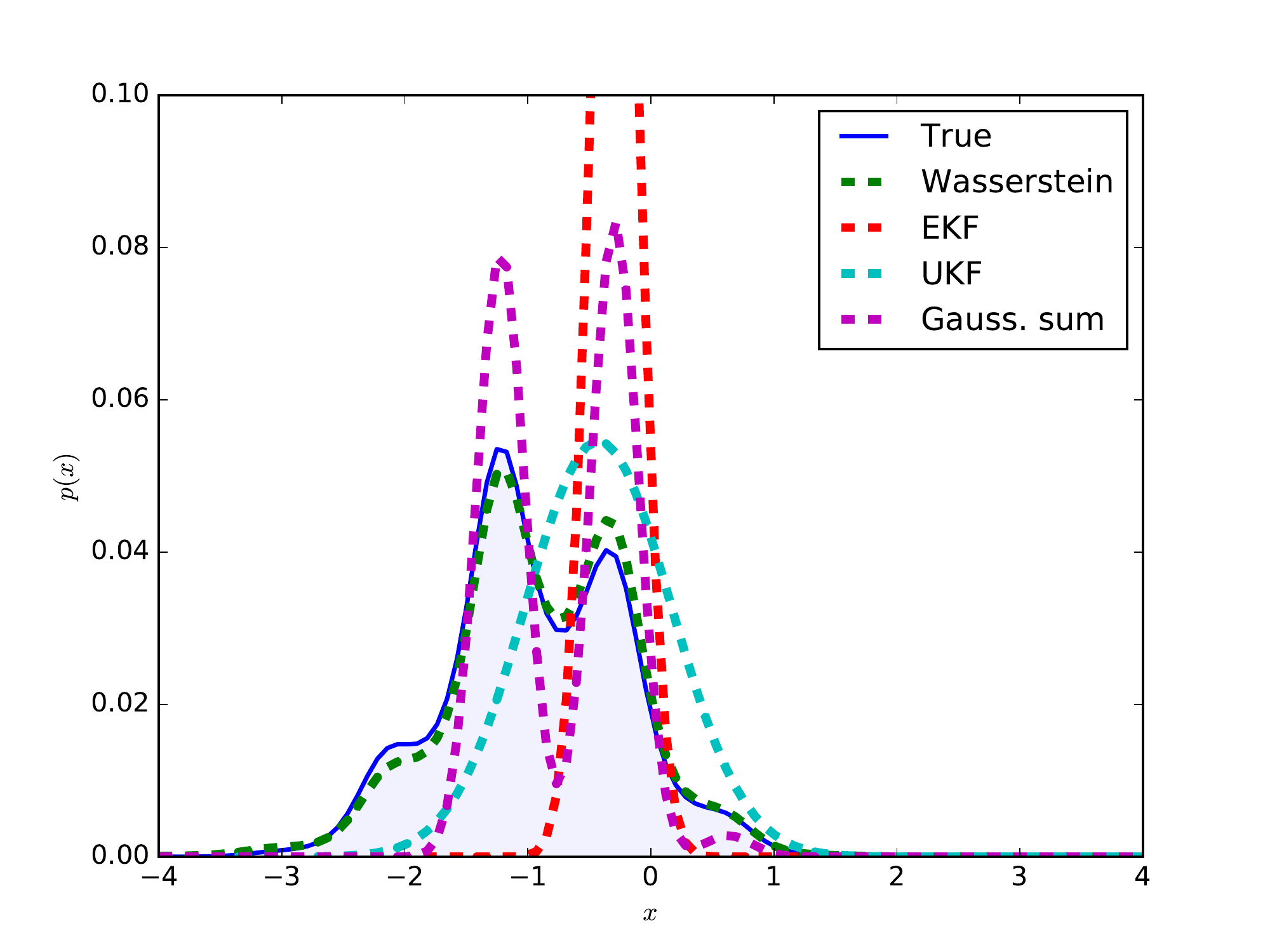}
  \caption{$t = 1$.}
  \label{fig:wassflow-empirical-filter-posterior-evolution-sine-1}
\end{subfigure}
\begin{subfigure}[t]{0.4\textwidth}
  \centering
  \includegraphics[width=\textwidth]{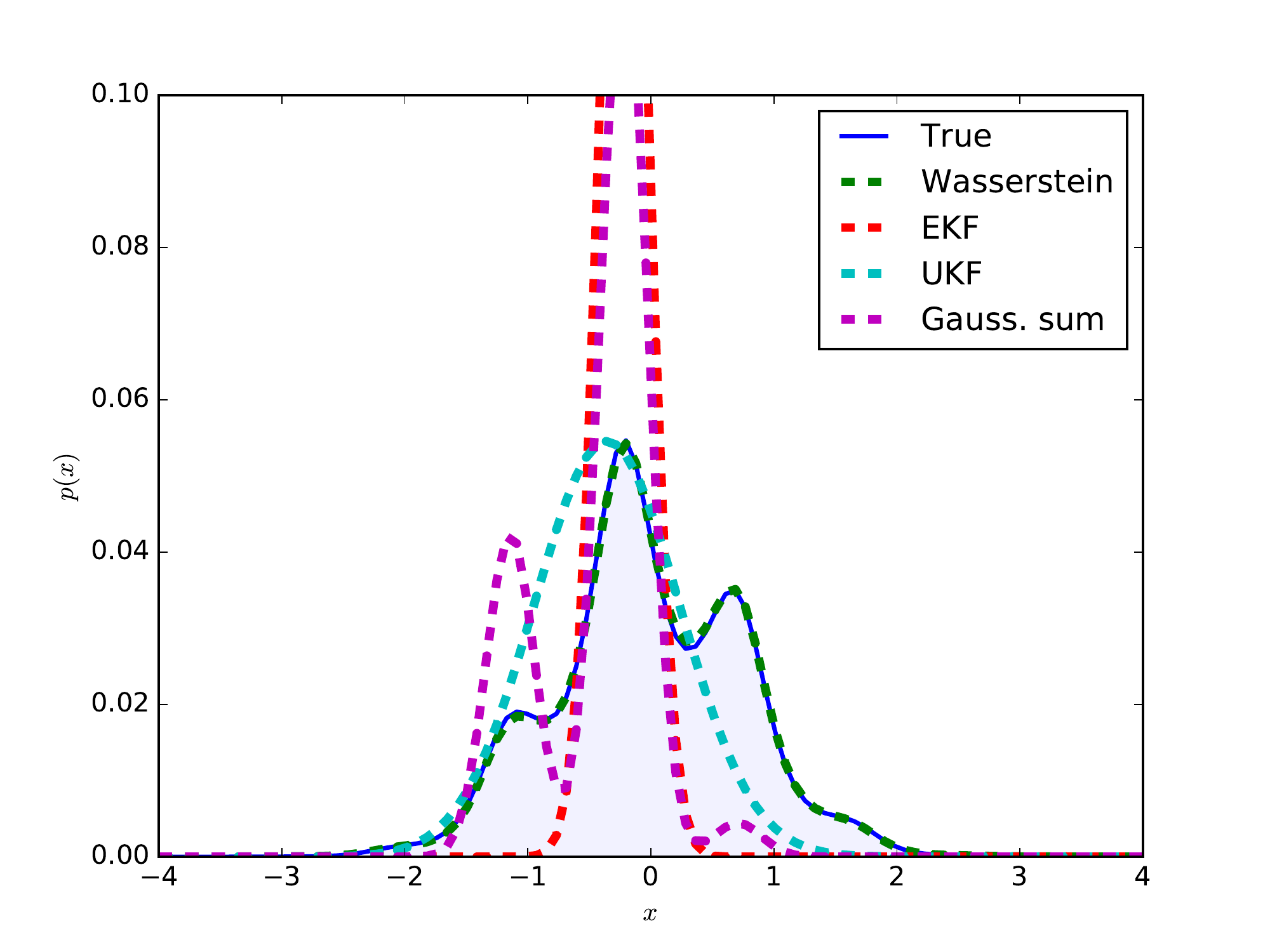}
  \caption{$t = 2$.}
  \label{fig:wassflow-empirical-filter-posterior-evolution-sine-2}
\end{subfigure}
\begin{subfigure}[t]{0.4\textwidth}
  \centering
  \includegraphics[width=\textwidth]{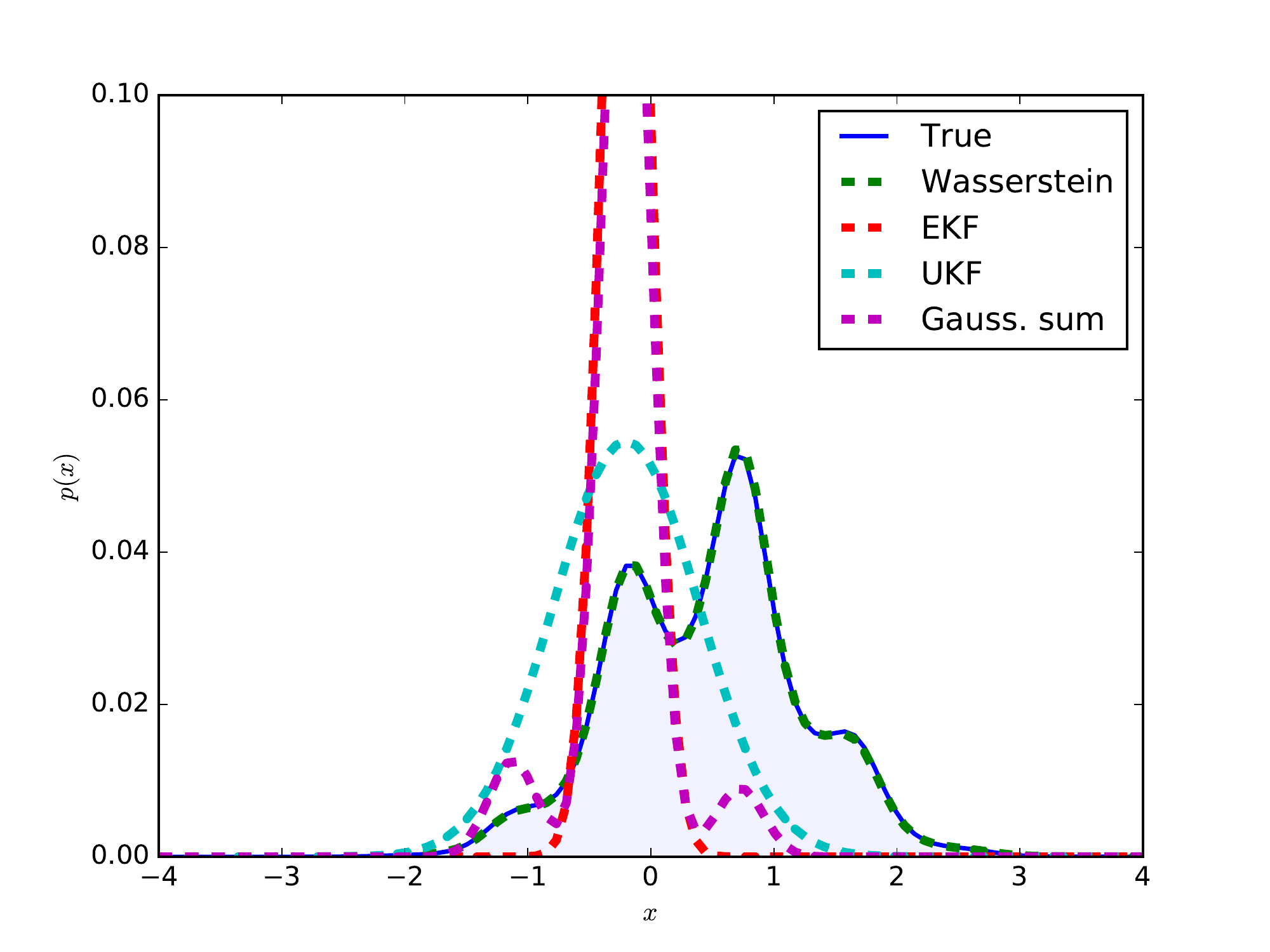}
  \caption{$t = 3$.}
  \label{fig:wassflow-empirical-filter-posterior-evolution-sine-3}
\end{subfigure}
\begin{subfigure}[t]{0.4\textwidth}
  \centering
  \includegraphics[width=\textwidth]{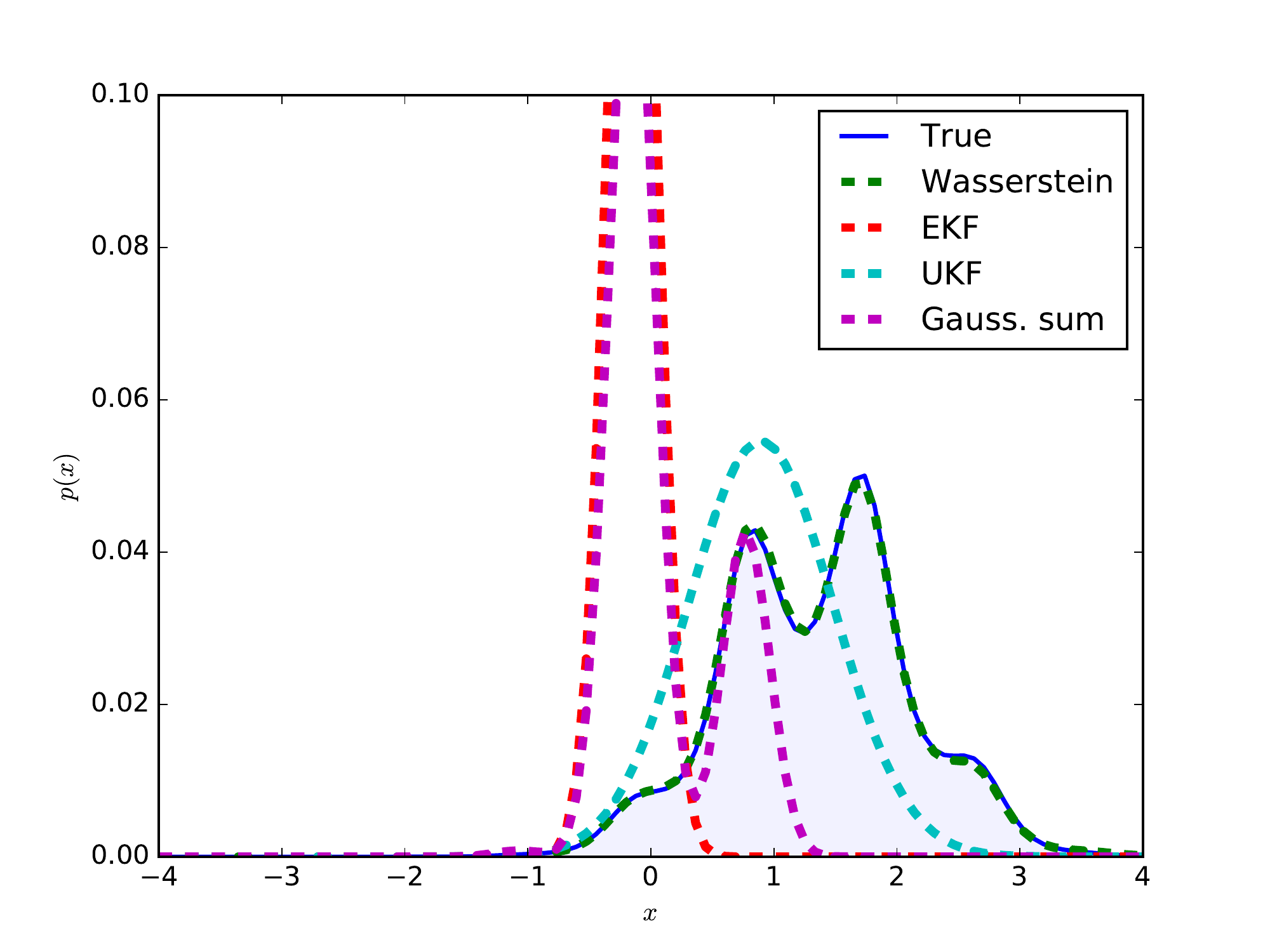}
  \caption{$t = 4$.}
  \label{fig:wassflow-empirical-filter-posterior-evolution-sine-4}
\end{subfigure}
\begin{subfigure}[t]{0.4\textwidth}
  \centering
  \includegraphics[width=\textwidth]{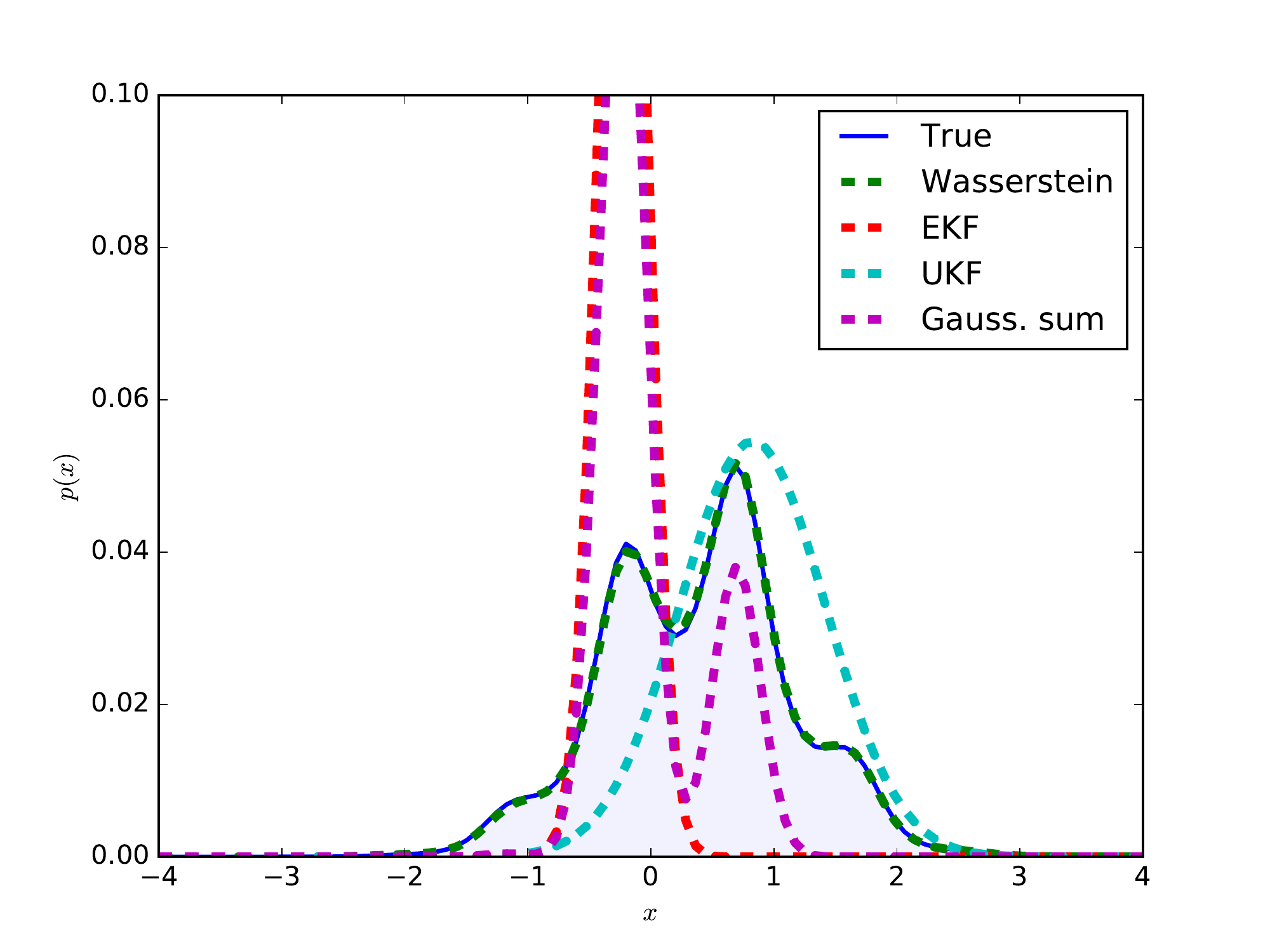}
  \caption{$t = 5$.}
  \label{fig:wassflow-empirical-filter-posterior-evolution-sine-5}
\end{subfigure}
\begin{subfigure}[t]{0.4\textwidth}
  \centering
  \includegraphics[width=\textwidth]{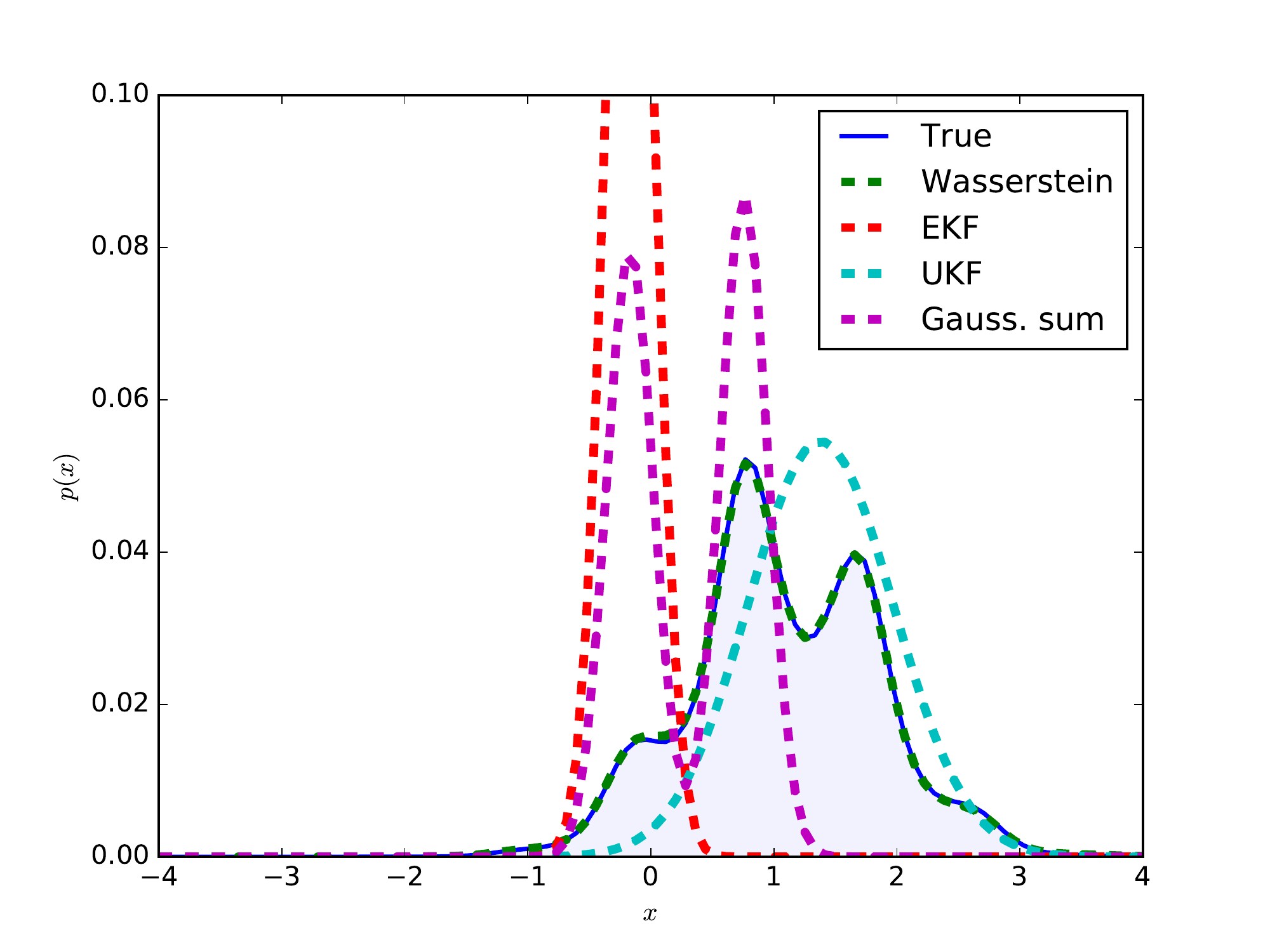}
  \caption{$t = 6$.}
  \label{fig:wassflow-empirical-filter-posterior-evolution-sine-6}
\end{subfigure}
\begin{subfigure}[t]{0.4\textwidth}
  \centering
  \includegraphics[width=\textwidth]{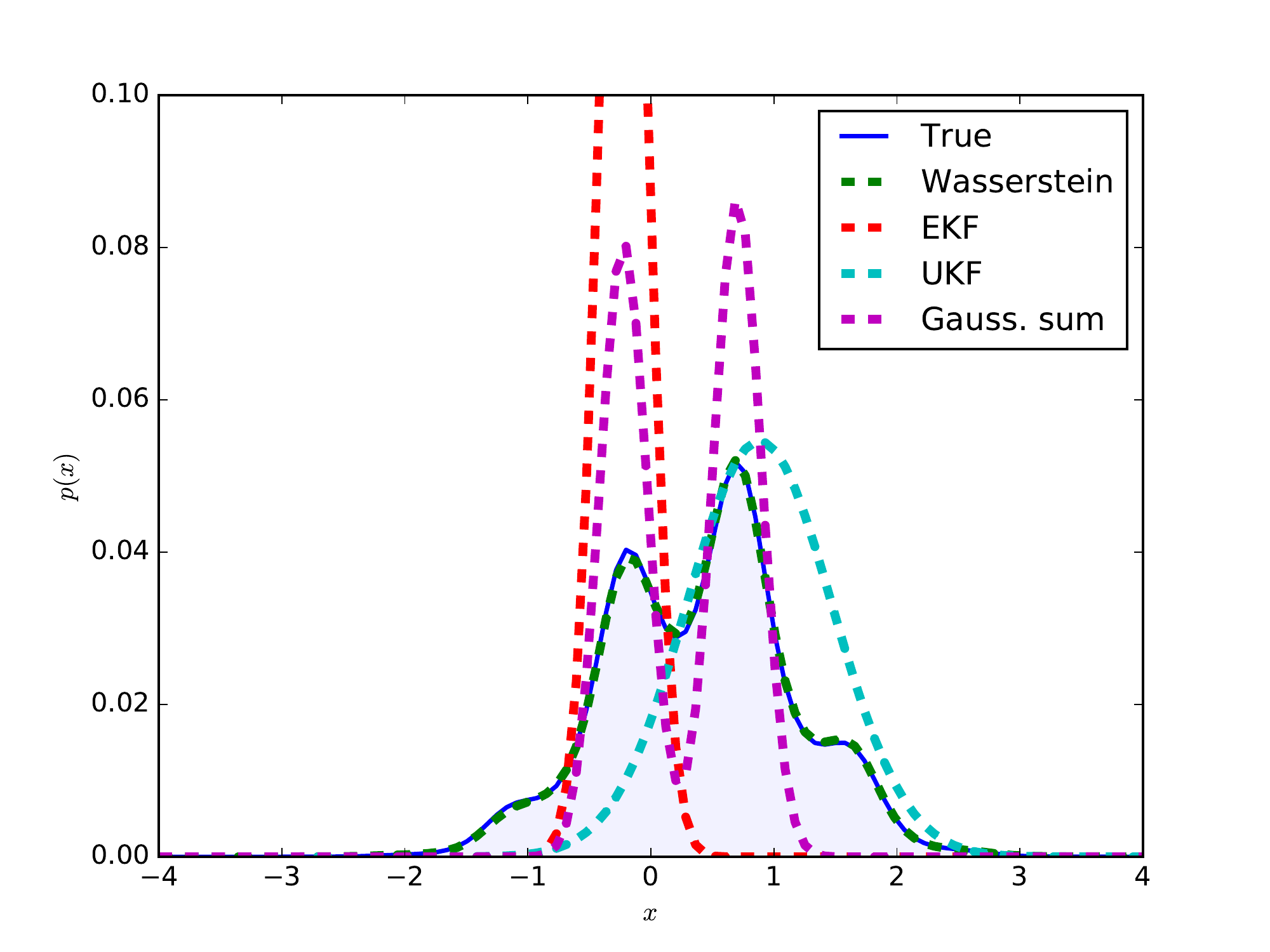}
  \caption{$t = 7$.}
  \label{fig:wassflow-empirical-filter-posterior-evolution-sine-7}
\end{subfigure}
\begin{subfigure}[t]{0.4\textwidth}
  \centering
  \includegraphics[width=\textwidth]{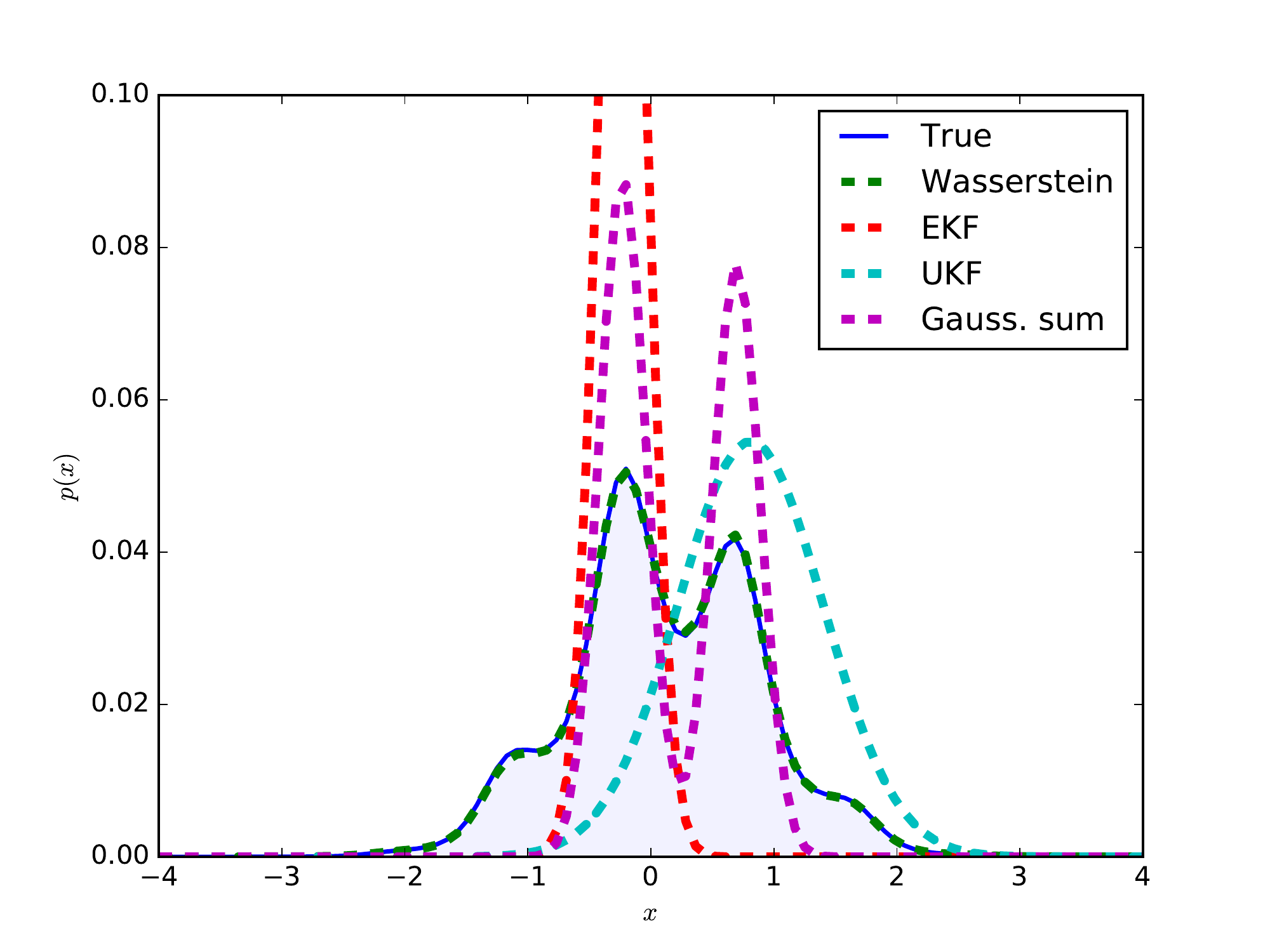}
  \caption{$t = 8$.}
  \label{fig:wassflow-empirical-filter-posterior-evolution-sine-8}
\end{subfigure}
\caption{Sine potential with noisy observations ($\sigma = 1$). Evolution of the posterior density, with estimates from the various methods overlaid. Shaded region is the exact solution.}
\label{fig:wassflow-empirical-filter-posterior-evolution-sine}
\end{figure}

\end{document}